%% file: arxiv.tex
\documentclass[10pt]{article}
\pdfoutput=1

\usepackage{etoolbox}
\newcommand{\arxiv}[1]{\iftoggle{colt}{}{#1}}
\newcommand{\colt}[1]{\iftoggle{colt}{#1}{}}
\newtoggle{colt}
\global\togglefalse{colt}

\colt{
\usepackage{times}
}

\usepackage[utf8]{inputenc} %
\usepackage[T1]{fontenc}    %
\usepackage{url}            %
\usepackage{booktabs}       %
\usepackage{amsfonts}       %
\usepackage{nicefrac}       %
\usepackage{microtype}      %

\usepackage{tocloft}            %

\usepackage{enumitem}

\usepackage{breakcites}

\newtoggle{draft}
\togglefalse{draft}

\usepackage{mathrsfs}

\usepackage{algorithm}
\usepackage{verbatim}
\usepackage[noend]{algpseudocode}

\usepackage{multicol}

\usepackage{colortbl}

\usepackage{setspace}

\usepackage{transparent}

\usepackage{inconsolata}
\usepackage[scaled=.90]{helvet}
\usepackage{xspace}

\usepackage{pifont}

\input{arxiv_style}

\input{dylan}

\input{macros}

\input{widebar}

\usepackage[suppress]{color-edits}
 \addauthor{ak}{ForestGreen}
 \addauthor{pa}{Lavender}

\arxiv{
\usepackage[final]{showlabels}

}

\makeatletter
\let\OldStatex\Statex
\renewcommand{\Statex}[1][3]{%
  \setlength\@tempdima{\algorithmicindent}%
  \OldStatex\hskip\dimexpr#1\@tempdima\relax}
\makeatother

\usepackage{accents}
\usepackage{wrapfig}
\usepackage{tikz}
\usetikzlibrary{decorations.pathreplacing}

\let\oldparagraph\paragraph
\renewcommand{\paragraph}[1]{\oldparagraph{#1.}}

\makeatletter
\newcommand{\tightparagraph}{%
  \@startsection{paragraph}{4}%
  {\z@}{1.25ex \@plus 1ex \@minus .2ex}{-1em}%
  {\normalfont\normalsize\bfseries}%
}
\makeatother

\usepackage{yfonts}

\arxiv{
  \title{Mitigating Covariate Shift in Misspecified Regression \\ with Applications to Reinforcement Learning\footnotetext{Authors listed in alphabetical order.}}
\author{%
Philip Amortila \\
{\normalsize University of Illinois, Urbana-Champaign}\\
{\small\texttt{philipa4@illinois.edu}}
\and
Tongyi Cao \\
{\normalsize University of Massachusetts, Amherst}\\
{\small\texttt{tcao@cs.umass.edu}}
\and
Akshay Krishnamurthy\\
{\normalsize Microsoft Research, NYC}\\
{\small\texttt{akshaykr@microsoft.com}}
}
\date{}
}

\begin{document}

\maketitle

\begin{abstract}

\input{abstract}

\end{abstract}

\section{Introduction}
\label{sec:intro}
\input{section_introduction}

\section{Misspecified regression under distribution shift}
\label{sec:regression}
\input{section_regression}

\section{Proof of~\pref{thm:main}}
\label{sec:proof_body}
\input{section_proof}

\section{Applications to online and offline reinforcement learning}
\label{sec:rl}
\input{section_rl}

\section{Related work}
\label{sec:related}
\input{section_related}

\section{Discussion}
\label{sec:discussion}
\input{section_discussion}

\subsection*{Acknowledgements}
We thank Adam Block for helpful feedback on a early version of the manuscript.

\clearpage

\bibliography{refs} 

\clearpage

\appendix  

\section{Proofs for~\pref{sec:regression}}
\label{app:regression}
\input{appendix_regression}

\section{Proofs for~\pref{sec:rl}}
\label{app:rl}
\input{appendix_rl}

\end{document}

%% file: arxiv_style.tex
\arxiv{
\usepackage[letterpaper, left=1in, right=1in, top=1in, bottom=1in]{geometry}
\PassOptionsToPackage{hypertexnames=false}{hyperref}  %
\usepackage{parskip}
\usepackage[dvipsnames]{xcolor}
\usepackage[colorlinks=true, linkcolor=blue!70!black, citecolor=blue!70!black,urlcolor=black,breaklinks=true]{hyperref}
}

\colt{
\PassOptionsToPackage{dvipsnames}{xcolor} 
}

\usepackage{microtype}
\usepackage{hhline}

\makeatletter
\newcommand{\neutralize}[1]{\expandafter\let\csname c@#1\endcsname\count@}
\makeatother

\usepackage{algorithm}

\arxiv{
\usepackage{natbib}
\bibliographystyle{plainnat}
\bibpunct{(}{)}{;}{a}{,}{,}
}

\usepackage{amsthm}
\usepackage{mathtools}
\usepackage{amsmath}
\usepackage{bbm}
\usepackage{amsfonts}
\usepackage{amssymb}

\usepackage{xpatch}

\usepackage{thmtools}
\usepackage{thm-restate}
\declaretheorem[name=Theorem,parent=section]{theorem}
\declaretheorem[name=Lemma,parent=section]{lemma}
\declaretheorem[name=Assumption, parent=section]{assumption}
\declaretheorem[name=Condition, parent=section]{condition}

\declaretheorem[name=Proposition, parent=section]{proposition}

\declaretheorem[name=Corollary, parent=section]{corollary}

\makeatletter
  \renewenvironment{proof}[1][Proof]%
  {%
   \par\noindent{\bfseries\upshape {#1.}\ }%
  }%
  {\qed\newline}
  \makeatother

\theoremstyle{definition}  %

\theoremstyle{plain}
\newtheorem{definition}{Definition}[section]

\xpatchcmd{\proof}{\itshape}{\normalfont\proofnameformat}{}{}
\newcommand{\proofnameformat}{\bfseries}

\usepackage[nameinlink,capitalize]{cleveref}

\newcommand{\pref}[1]{\cref{#1}}
\newcommand{\pfref}[1]{Proof of \pref{#1}}

\renewcommand{\eqref}[1]{\texorpdfstring{\hyperref[#1]{(\ref*{#1})}}{(\ref*{#1})}}

\crefformat{equation}{#2Eq. (#1)#3}
\Crefformat{equation}{#2Eq. (#1)#3}

\Crefformat{figure}{#2Figure #1#3}
\Crefname{assumption}{Assumption}{Assumptions}
\Crefformat{assumption}{#2Assumption #1#3}
\Crefname{subsubsection}{Section}{Sections}
\crefformat{subsubsection}{#2Section #1#3}
\Crefformat{subsubsection}{#2Section #1#3}

\usepackage{crossreftools}
\usepackage{xparse}

\ExplSyntaxOn
\DeclareDocumentCommand{\XDeclarePairedDelimiter}{mm}
 {
  \__egreg_delimiter_clear_keys: %
  \keys_set:nn { egreg/delimiters } { #2 }
  \use:x %
   {
    \exp_not:n {\NewDocumentCommand{#1}{sO{}m} }
     {
      \exp_not:n { \IfBooleanTF{##1} }
       {
        \exp_not:N \egreg_paired_delimiter_expand:nnnn
         { \exp_not:V \l_egreg_delimiter_left_tl }
         { \exp_not:V \l_egreg_delimiter_right_tl }
         { \exp_not:n { ##3 } }
         { \exp_not:V \l_egreg_delimiter_subscript_tl }
       }
       {
        \exp_not:N \egreg_paired_delimiter_fixed:nnnnn 
         { \exp_not:n { ##2 } }
         { \exp_not:V \l_egreg_delimiter_left_tl }
         { \exp_not:V \l_egreg_delimiter_right_tl }
         { \exp_not:n { ##3 } }
         { \exp_not:V \l_egreg_delimiter_subscript_tl }
       }
     }
   }
 }

\keys_define:nn { egreg/delimiters }
 {
  left      .tl_set:N = \l_egreg_delimiter_left_tl,
  right     .tl_set:N = \l_egreg_delimiter_right_tl,
  subscript .tl_set:N = \l_egreg_delimiter_subscript_tl,
 }

\cs_new_protected:Npn \__egreg_delimiter_clear_keys:
 {
  \keys_set:nn { egreg/delimiters } { left=.,right=.,subscript={} }
 }

\cs_new_protected:Npn \egreg_paired_delimiter_expand:nnnn #1 #2 #3 #4
 {%
  \mathopen{}
  \mathclose\c_group_begin_token
   \left#1
   #3
   \group_insert_after:N \c_group_end_token
   \right#2
   \tl_if_empty:nF {#4} { \c_math_subscript_token {#4} }
 }
\cs_new_protected:Npn \egreg_paired_delimiter_fixed:nnnnn #1 #2 #3 #4 #5
 {
  \mathopen{#1#2}#4\mathclose{#1#3}
  \tl_if_empty:nF {#5} { \c_math_subscript_token {#5} }
 }
\ExplSyntaxOff

\XDeclarePairedDelimiter{\supnorm}{
  left=\lVert,
  right=\rVert,
  subscript=\infty
  }

%% file: dylan.tex
\DeclarePairedDelimiter{\abs}{\lvert}{\rvert} %
\DeclarePairedDelimiter{\brk}{[}{]}
\DeclarePairedDelimiter{\crl}{\{}{\}}
\DeclarePairedDelimiter{\prn}{(}{)}
\DeclarePairedDelimiter{\nrm}{\|}{\|}

\let\Pr\undefined

\DeclareMathOperator{\Pr}{Pr}

\DeclareMathOperator*{\argmin}{arg\,min} %
\DeclareMathOperator*{\argmax}{arg\,max}

\def\ddefloop#1{\ifx\ddefloop#1\else\ddef{#1}\expandafter\ddefloop\fi}
\def\ddef#1{\expandafter\def\csname bb#1\endcsname{\ensuremath{\mathbb{#1}}}}
\ddefloop ABCDEFGHIJKLMNOPQRSTUVWXYZ\ddefloop
\def\ddefloop#1{\ifx\ddefloop#1\else\ddef{#1}\expandafter\ddefloop\fi}
\def\ddef#1{\expandafter\def\csname b#1\endcsname{\ensuremath{\mathbf{#1}}}}
\ddefloop ABCDEFGHIJKLMNOPQRSTUVWXYZ\ddefloop
\def\ddef#1{\expandafter\def\csname sf#1\endcsname{\ensuremath{\mathsf{#1}}}}
\ddefloop ABCDEFGHIJKLMNOPQRSTUVWXYZ\ddefloop
\def\ddef#1{\expandafter\def\csname c#1\endcsname{\ensuremath{\mathcal{#1}}}}
\ddefloop ABCDEFGHIJKLMNOPQRSTUVWXYZ\ddefloop
\def\ddef#1{\expandafter\def\csname h#1\endcsname{\ensuremath{\widehat{#1}}}}
\ddefloop ABCDEFGHIJKLMNOPQRSTUVWXYZ\ddefloop
\def\ddef#1{\expandafter\def\csname hc#1\endcsname{\ensuremath{\widehat{\mathcal{#1}}}}}
\ddefloop ABCDEFGHIJKLMNOPQRSTUVWXYZ\ddefloop
\def\ddef#1{\expandafter\def\csname t#1\endcsname{\ensuremath{\widetilde{#1}}}}
\ddefloop ABCDEFGHIJKLMNOPQRSTUVWXYZ\ddefloop
\def\ddef#1{\expandafter\def\csname tc#1\endcsname{\ensuremath{\widetilde{\mathcal{#1}}}}}
\ddefloop ABCDEFGHIJKLMNOPQRSTUVWXYZ\ddefloop
\def\ddefloop#1{\ifx\ddefloop#1\else\ddef{#1}\expandafter\ddefloop\fi}
\def\ddef#1{\expandafter\def\csname scr#1\endcsname{\ensuremath{\mathscr{#1}}}}
\ddefloop ABCDEFGHIJKLMNOPQRSTUVWXYZ\ddefloop

\newcommand{\ind}{\mathbbm{1}}    %

\newcommand{\veps}{\varepsilon}

%% file: macros.tex
\newcommand{\Var}{\mathrm{Var}}
\newcommand{\poly}{\mathrm{poly}}
\renewcommand{\index}[1]{^{\scriptscriptstyle(#1)}} %

\newcommand{\dtest}{d_{\mathrm{test}}}
\newcommand{\dtrain}{d_{\mathrm{train}}}
\newcommand{\Dtrain}{\cD_{\mathrm{train}}}
\newcommand{\Dtest}{\cD_{\mathrm{test}}}
\newcommand{\Ptrain}{\bbP_{\mathrm{train}}}
\newcommand{\Ptest}{\bbP_{\mathrm{test}}}
\newcommand{\Etrain}{\bbE_{\mathrm{train}}}

\newcommand{\Etest}{\bbE_{\mathrm{test}}}
\newcommand{\Rtrain}{R_{\mathrm{train}}}

\newcommand{\Rhat}{\widehat{R}}
\newcommand{\Rtest}{R_{\mathrm{test}}}

\newcommand{\fstar}{f^{\star}}
\newcommand{\fbar}{\bar{f}}
\newcommand{\fhat}{\hat{f}}
\newcommand{\ferm}{\fhat^{(n)}_{\normalfont \erm}}
\newcommand{\ferminf}{\fhat^{(\infty)}_{\normalfont \erm}}
\newcommand{\finf}{\fhat^{(n)}_{\infty}}
\newcommand{\fdro}{\fhat^{(n)}_{\normalfont \dbr}}

\newcommand{\fbad}{f_{\mathrm{bad}}}
\newcommand{\ftrg}{f_{\mathrm{trg}}}
\newcommand{\apx}{\mathsf{apx}}

\newcommand{\misspec}{\veps_{\infty}}
\newcommand{\epsstat}{\veps_{\mathrm{stat}}}
\newcommand{\epsslow}{\veps_{\mathrm{slow}}}

\newcommand{\Cinf}{C_{\infty}^{\vphantom{k}}}

\newcommand{\cLhat}{\widehat{\cL}}

\newcommand{\Ccov}{C_{\mathrm{cov}}}
\newcommand{\Cconc}{C_{\mathrm{conc}}}
\newcommand{\Ctrans}{C_{\mathrm{transfer}}}
\newcommand{\pit}{\pi\index{t}}
\newcommand{\piht}{\pi_h\index{t}}
\newcommand{\reg}{\mathrm{Reg}}
\newcommand{\err}{\mathrm{err}}

\newcommand{\golf}{{\normalfont \textsf{GOLF}}\xspace}
\newcommand{\golfdbr}{{\normalfont \textsf{GOLF.DBR}}\xspace}
\newcommand{\dbr}{{\normalfont \textsf{DBR}}\xspace}
\newcommand{\erm}{{\normalfont \textsf{ERM}}\xspace}

\newcommand{\leqlab}[1]{\mathmakebox[1.35em]{\buildrel #1 \over \leq}}

%% file: abstract.tex
A pervasive phenomenon in machine learning applications is
\emph{distribution shift}\paedit{,} where training and deployment conditions for
a machine learning model differ. As distribution shift typically
results in a degradation in performance, much attention has been
devoted to algorithmic interventions that mitigate these detrimental
effects. In this paper, we study 
the effect of distribution shift in the presence of model
misspecification, specifically focusing on
$L_{\infty}$-misspecified regression and \emph{adversarial covariate shift}, where the
regression target remains fixed while the covariate distribution
changes arbitrarily. 
We show that empirical risk minimization, or
standard least squares regression, can result in
undesirable \emph{misspecification amplification} where the
error due to misspecification is amplified by the density ratio between the
training and testing distributions. As our main result, we develop a
new algorithm---inspired by robust optimization techniques---that
avoids this undesirable behavior, resulting in no misspecification
amplification while still obtaining optimal statistical rates. As
applications, we use this regression procedure to obtain new
guarantees in offline and online reinforcement learning with
misspecification and establish new separations between previously studied
structural conditions and notions of coverage.

%% file: section_introduction.tex
A majority of machine learning methods are developed and analyzed under the idealized setting where the training conditions accurately reflect those at deployment. 
Yet, almost all practical applications exhibit \emph{distribution shift}, where these conditions differ significantly. 
Distribution shift can occur for a plethora of reasons, ranging from quirks in data collection~\citep{recht2019imagenet}, to temporal drift~\citep{gama2014survey,besbes2015non}, to users adapting to an ML model~\citep{perdomo2020performative}, and it typically results in a degradation in model performance. 
Due to the prevalence of this phenomenon and the diversity of applications where it manifests, %
there is a vast and ever-growing body of literature studying algorithmic interventions to mitigate distribution shift~\citep{quinonero2008dataset,sugiyama2012machine}.

\emph{Covariate shift}
is perhaps the most basic form of distribution shift. %
Covariate shift is pertinent to supervised learning---where the goal is to predict a label $Y$ from covariates $X$---and posits a change in the distribution over covariates while keeping the target predictor fixed. 
This setup, in particular that the target does not change, is natural in applications including neural algorithmic reasoning~\citep{anil2022exploring,zhang2022unveiling,liu2023exposing}, reinforcement learning~\citep{ross2011reduction,levine2020offline}, and computer vision~\citep{koh2021wilds,recht2019imagenet,miller2021accuracy}. 
It is well known that one can adapt guarantees from statistical learning to the covariate shift setting; specifically, for well-specified regression, a classical density-ratio argument shows that empirical risk minimization (ERM) is consistent under suitably well-behaved covariate shifts.

One stipulation of this consistency guarantee is that the model/hypothesis class be \textit{well-specified} (also referred to as \textit{realizable}).  
Although statistical learning theory offers a rather complete understanding of misspecification in the absence of covariate shift (via agnostic learning and excess risk bounds),
our understanding of how covariate shift can adversely interact with model misspecification remains fairly immature. 
This interaction is the focus of the present paper.

\subsection{Contributions}
We study regression under \emph{adversarial covariate shift} where we receive regression samples from a distribution $\Dtrain$ but are evaluated on an arbitrary distribution $\Dtest$ for which no prior knowledge is available; we only assume that the distributions share the same target regression function $\fstar$ and that the worst-case density ratio of the covariate marginals is bounded by $ \Cinf\in[1,\infty)$ (formally defined in \cref{sec:regression}). %
As inductive bias, we have a function class $\cF$ of predictors and assume \emph{$L_{\infty}$-misspecification}: there exists a predictor $\fbar\in\cF$ that is pointwise close to $\fstar$, i.e., $\|\fbar - \fstar \|_{\infty} \leq \misspec$. 
This notion is natural for the covariate shift setting because it ensures that $\fbar$ has low and comparable prediction error on both $\Dtrain$ and any $\Dtest$. %

In this setup we obtain the following results:
\begin{enumerate}
    \item We show that standard empirical risk minimization (\erm) is not robust to covariate shift in the presence of misspecification. 
    Precisely, even in the limit of infinite data, \erm over $\cF$ can incur squared prediction error under $\Dtest$ scaling as $\Omega(\Cinf \misspec^2)$. Meanwhile
    the error of the $L_{\infty}$-misspecified predictor $\fbar$ is at most $\misspec^2$. 
    We call this phenomenon---where the misspecification error is scaled by the density ratio coefficient (despite there being a predictor avoiding this scaling)---\emph{misspecification amplification}.
    
    \item As our main result, we give a new algorithm, called disagreement-based regression (\dbr), that avoids \emph{misspecification amplification} and is therefore robust to adversarial covariate shift under misspecification.
    \dbr has asymptotic prediction error under $\Dtest$ scaling as $O(\misspec^2)$, with no dependence on the density ratio coefficient $\Cinf$. 
    At the same time, it has order-optimal finite sample behavior recovering standard ``fast rate'' guarantees for the well-specified setting, and can be extended to adapt to unknown misspecification level (as shown in~\pref{app:extensions}). 
    To our knowledge, this is the first result avoiding misspecification amplification in the adversarial covariate shift setting. 
    Our assumptions---particularly that no information about $\Dtest$ is available and that $\cF$ is unstructured---rule out prior approaches based on density ratios~\citep{shimodaira2000improving,duchi2021learning} or sup-norm convergence~\citep{schmidt2022local}; see~\pref{sec:related} for further discussion.

\end{enumerate}

To demonstrate the utility of disagreement-based regression, we deploy the procedure in value function approximation settings in reinforcement learning (RL), where regression is a standard primitive and mitigating the adverse effects of distribution shift is a central challenge. 
Here, using \dbr as a drop-in replacement for \erm when fitting Bellman backups, we obtain the following results:
\begin{enumerate}
    \item In the offline RL setting, we instantiate the minimax algorithm of~\citet{chen2019information} with \dbr and show that, under $L_{\infty}$-misspecification and with coverage measured via the concentrability coefficient, misspecification amplification can be avoided when learning a near optimal policy. 
    In contrast, prior lower bounds imply that misspecification amplification is unavoidable when coverage is measured via Bellman transfer coefficients~\citep{du2019good,van2019comments,lattimore2020learning}.
    Our result therefore establishes a new separation between concentrability and Bellman transfer coefficients. 
    \item In the online RL setting, we instantiate the \golf algorithm of~\citet{jin2021bellman} with \dbr and obtain analogous results under the structural condition of \emph{coverability} (building on the analysis of~\citet{xie2022role}).
    Taken with the above lower bounds~\citep{du2019good,van2019comments,lattimore2020learning}, %
this separates structural conditions involving Bellman errors (e.g., Bellman rank~\citep{jiang2017contextual}, Bellman-eluder dimension~\citep{jin2021bellman}, or sequential extrapolation coefficient~\citep{xie2022role}) from coverability, which does not.
\end{enumerate}

To keep the presentation concise and focused on the interaction between covariate shift and misspecification, we focus on the simplest settings that manifest misspecification amplification. 
In~\pref{sec:discussion}, we discuss a number of directions for future work, which include extensions to the core technical and algorithmic results.

%% file: section_regression.tex
We begin by introducing the formal problem setting and our assumptions. 
Most proofs for results in this section are deferred to \cref{app:regression}. %
There are two joint distributions, called $\Dtrain$ and $\Dtest$, over $\cX\times\bbR$ where $\cX$ is a covariate space. %
We use $\Ptrain,\Ptest$ and $\Etrain,\Etest$ to denote the probability law and expectation under these distributions. 
We hypothesize that $\Dtrain$ and $\Dtest$ share the same \emph{Bayes regression function}, an assumption referred to as covariate shift in the literature~\citep{shimodaira2000improving}.
\begin{assumption}[Covariate shift]
\label{assum:covariate_shift}
For all $x \in \cX$ we have 
\begin{align*}
\Etrain[y \mid x] = \Etest[y \mid x].
\end{align*}
\end{assumption}
Let $\fstar: x \mapsto \Etrain[y \mid x]$ denote the shared Bayes regression function.
We posit that the marginal distributions over $\cX$ are absolutely continuous with respect to a reference measure and use $\dtrain$ and $\dtest$ to denote the corresponding marginal densities.
We assume these are related via the following density ratio assumption.
\begin{assumption}[Bounded density ratios]
\label{assum:density_ratios}
The density ratio 
\[
\Cinf := \sup_{x\in\cX} \abs*{ \frac{\dtest(x)}{\dtrain(x)} }
\]
is bounded, i.e., $\Cinf < \infty$.
\end{assumption}
Note that $\Cinf \geq 1$ always. 
Boundedness of density ratios is standard in the covariate shift literature; indeed the coefficient $\Cinf$ appears in the classical covariate shift analyses as well as in many algorithmic interventions~\citep{shimodaira2000improving,sugiyama2007direct}.
Beyond satisfying these assumption, $\Dtest$ can be adaptively and adversarially chosen. 
In particular, no information about $\Dtest$, such as labeled/unlabeled samples or other inductive bias, is available. 

We have a dataset $\{(x_i,y_i)\}_{i=1}^n$ of $n$ i.i.d. labeled examples sampled from $\Dtrain$ and a function class $\cF \subset (\cX \to \bbR)$ of predictors. We define the (squared) \emph{prediction errors}
\begin{align}
\Rtrain(f): = \Etrain\brk*{ (f(x) - \fstar(x))^2 }, \quad \mathrm{and} \quad  \Rtest(f) := \Etest\brk*{ (f(x) - \fstar(x))^2 }. \label{eq:test_obj}
\end{align}
We seek to use the dataset to find a predictor $\fhat$ for which $\Rtest(\fhat)$ is small.

Regarding $\cF$, we make two assumptions: we assume that $|\cF|<\infty$ and that $\cF$ is $L_{\infty}$-misspecified.
\begin{assumption}[$L_{\infty}$-misspecification]
\label{assum:misspecification}
For some $\misspec \geq 0$, there exists $\fbar \in \cF$ with 
\begin{align*}
\nrm*{\fbar - \fstar}_{\infty} \leq \misspec, \quad \mathrm{where} \quad \nrm*{f}_{\infty} := \sup_{x\in\cX} | f(x) |.
\end{align*}
\end{assumption}
Most prior analyses for regression under covariate shift assume that the model class $\cF$ is well-specified, i.e., that $\veps_{\infty}=0$ so that $\fstar \in \cF$. 
$L_{\infty}$-misspecification provides a relaxation that is natural for at least two reasons.
First, it enables end-to-end learning guarantees via composition with approximation-theoretic results for specific function classes (e.g., neural networks), where it is standard to measure approximation via the $L_{\infty}$ norm~\citep{mjt_dlt}.
More importantly, $L_{\infty}$-misspecification is particularly apt in the covariate shift setting because it ensures that $\fbar$ has low prediction error on both $\Dtest$ and $\Dtrain$. 
Thus, there is at least one high-quality predictor whose performance is stable across distributions.
In contrast, we have no such guarantee if we, for example, measure misspecification with respect to other norms (which depend on the distribution) or consider the agnostic setting (with no quantified misspecification assumption).
Indeed, we will see below that misspecification amplification is unavoidable in such cases.

We also make the following technical assumption.
\begin{assumption}[Boundedness]
\label{assum:boundedness}
$\sup_{f \in \cF} \nrm*{f}_{\infty} \leq 1$ and $\abs*{y} \leq 1$ almost surely under $\Dtrain$ and $\Dtest$. 
\end{assumption}
We impose~\pref{assum:boundedness} and that $|\cF| < \infty$ solely to highlight the novel algorithmic and technical aspects; we expect that relaxing these assumptions is possible.

\subsection{Misspecification amplification for empirical risk minimization}\label{sec:erm-amplification}
When there is no prior knowledge about or data from $\Dtest$, perhaps the most natural algorithm for optimizing $\Rtest(\cdot)$ is empirical risk minimization (\erm) on the data from the training distribution:
\begin{align*}
\ferm := \argmin_{f \in \cF} \frac{1}{n}\sum_{i=1}^n (f(x_i) - y_i)^2. 
\end{align*}

A standard uniform convergence argument yields the classical covariate shift guarantee for \erm:
\begin{restatable}[\erm upper bound]{proposition}{ermub}
\label{prop:erm_ub}
For any $\delta \in (0,1)$ with probability at least $1-\delta$,  \erm satisfies
\begin{align*}
\Rtest(\ferm) \leq O\prn*{\Cinf \misspec^{2} +  \Cinf \frac{ \log (|\cF|/\delta)}{n}}.%
\end{align*}
\end{restatable}
The second term which scales as $1/n$---the statistical term---is optimal in the generality of our setup~\citep{ma2023optimally,ge2023maximum}, the interpretation being that the effective sample size is reduced by a factor of $\Cinf$ due to the mismatch between $\Dtrain$ and $\Dtest$. %
The first term---the misspecification term---represents the asymptotic\footnote{We consider the asymptotic regime where $n \to \infty$ with all other quantities, like $\log |\cF|$ and $\misspec$, fixed.} test error of \erm and demonstrates a phenomenon that we call \emph{misspecification amplification}, whereby the error due to misspecification is amplified by the density ratio coefficient. 
This phenomenon is simultaneously more concerning and less intuitive than the degradation of the statistical term, because it describes an error which does not decay with larger sample sizes and because $\fbar \in \cF$ has $\Rtest(\fbar) = \misspec^2$.
Since $\cF$ contains a predictor that does not incur misspecification amplification, one might hope that misspecification amplification can be avoided. 

Our first main result is that misspecification amplification \emph{cannot} be avoided by \erm in the worst case. 
The result is proved in the asymptotic regime, where \erm is equivalent to the $L_2(\Dtrain)$-projection of $\fstar$ onto the function class $\cF$, defined as
\begin{align*}
	\ferminf \in \argmin_{f \in \cF} \| f - \fstar\|_{L_2(\Dtrain)}^2, \qquad \mathrm{ with } \qquad \| g \|_{L_2(\Dtrain)}^2 := \Etrain\brk*{ g(x)^2 }.
\end{align*}
The next proposition shows that $\ferminf$ can incur misspecification amplification.
\begin{restatable}[\erm lower bound]{proposition}{ermlb}
\label{prop:erm_lb}
For all $\misspec \in (0,1)$ and $\Cinf\in[1,\infty)$ such that $\sqrt{\Cinf}\cdot\misspec \leq 1/2$, and for all $\zeta > 0$        sufficiently small,
there exist distributions $\Dtrain,\Dtest$ and a function class $\cF$ with $|\cF| = 2$ satisfying~\pref{assum:covariate_shift}-\pref{assum:boundedness} (with parameters $\misspec,\Cinf$) such that
\begin{align*}
	\Rtest(\ferminf) = \Cinf \misspec^2 - \zeta.
\end{align*}
\end{restatable}

\begin{wrapfigure}{R}{0.33\textwidth}
\vspace{-0.5cm}
\begin{center}
\hspace{0.35cm}\input{lb_fig.tex}
\end{center}
\vspace{-0.35cm}
\caption{The construction used to prove~\pref{prop:erm_lb}. $\fbad$ and $\fbar$ have equal risk under $\Dtrain$ but $\fbad$ concentrates errors onto $\Dtest$.}
\vspace{-0.25cm}
\label{fig:lb}
\end{wrapfigure}

Combined with the optimality of the statistical term~\citep{ma2023optimally,ge2023maximum}, this establishes that~\pref{prop:erm_ub} characterizes the behavior of \erm under $L_{\infty}$-misspecification and covariate shift. 
The construction is based on the following insight, visualized in~\pref{fig:lb}.
The fact that $\fbar$ is $L_{\infty}$-close to $f^\star$ guarantees that its prediction errors are ``spread out'' across the domain $\cX$.
Since $\fbar \in \cF$, we know that $\ferminf$ must satisfy $\nrm{\ferminf - \fstar}_{L_2(\Dtrain)}^2 \leq \misspec^2$. %
Unfortunately, this property does not  guarantee that the errors of $\ferminf$ are ``spread out'' in a similar manner to $\fbar$'s.
Indeed, we construct a predictor $\fbad$ that concentrates its errors on a region of $\cX$ that is amplified by $\Dtest$ and makes up for this by having zero error elsewhere.
By setting the parameters carefully, we can ensure that this bad predictor is chosen by \erm.

We note that essentially the same construction shows that, under the weaker notion of $L_2(\Dtrain)$-misspecification, amplification is unavoidable for any proper learner (which outputs a function in $\cF$). 
Indeed, in~\pref{fig:lb}, the function class $\{\fbad\}$ is
$L_2(\Dtrain)$-misspecified but $\fbad$ has much higher error on $\Dtest$.

\paragraph{Other existing algorithms}
\pref{prop:erm_lb} only pertains to \erm, and thus, one might ask whether other algorithms can avoid misspecification amplification. 
Before turning to our positive results in the next section, we briefly note that other standard algorithms (that do not require knowledge of $\Dtest$) either incur misspecification amplification to some degree, or have some other failure mode. 
This pertains to the star algorithm~\citep{audibert2007progressive,liang2015learning}, other aggregation schemes~\citep[c.f.,][]{lecue2014optimal}, and $L_{\infty}$-regression~\citep{knight2017asymptotic,yi2024non}, as we discuss in~\pref{app:other_algs}.
Several methods for mitigating covariate shift can avoid misspecification amplification, but either require knowledge of $\Dtest$ or structural assumptions on $\cF$; see~\pref{sec:related}.

\subsection{Main result: Disagreement-based regression}
In this section, we provide a new algorithm that avoids misspecification amplification while requiring no knowledge of $\Dtest$ and recovering optimal statistical rates. 
To develop some intuition, 
observe that in the construction in~\pref{fig:lb}, the only way for the bad predictor ($\fbad$, in red) to be chosen by \erm \emph{and} have large errors on $\Dtest$ is for it to have much lower error than $\fbar$ on the rest of the domain.
Indeed, if we could filter out the points where $\fbad$'s error is less than $\fbar$'s, then $\fbad$ cannot overcome the large errors on $\Dtest$. 
Stated another way, we can avoid misspecification amplification in this example if we restrict the regression problem to the region where $|\fbad(x) - \fstar(x)| \geq |\fbar(x) - \fstar(x)|$.

Generalizing this insight to a larger function class suggests that, when considering a candidate $f \in \cF$, we should only measure the square loss for $f$ on the region where $|f(x) - f^\star(x)| \geq |\fbar(x) - \fstar(x)|$. 
Unfortunately, this region depends on $\fstar$ and $\fbar$, both of which are unknown.
Nevertheless, our approach is based on this intuition, and we avoid the dependence on these unknown functions with two algorithmic ideas.

To eliminate the dependence on $\fstar$, we use the fact that $|\fbar(x) - \fstar(x)| \leq \misspec$ and approximate the above region with $I_f := \{ x : |f(x) - \fbar(x)| \geq c\misspec \}$. 
Indeed for $c \geq 2$,
\begin{align*}
 \{ x : |f(x) - \fbar(x)| \geq c\misspec \} \subseteq \{ x: |f(x) - \fstar(x)| \geq |\fbar(x) - \fstar(x)| \}.
\end{align*}
On the other hand, we know that $|f(x) - \fstar(x)| \leq (c+1)\misspec$ in the complementary region, $I_f^C$. This is, up to the constant factor, the best pointwise guarantee we can attain, making it safe to ignore the complementary region. This resolves the first issue of dependence on $\fstar$.

To address the dependence on $\fbar$, we use that $\fbar \in \cF$ and formulate a robust optimization objective that implicitly considers all possible pairwise ``disagreement regions.'' 
Formally, the algorithm is:
\begin{align}
W_{f,g}^\tau(x) := \ind\{ |f(x) - g(x)| \geq \tau \}, \quad \fdro \gets \argmin_{f \in \cF} \max_{g \in \cF} \frac{1}{n}\sum_{i=1}^n W_{f,g}^\tau(x) \crl*{(f(x) - y)^2 - (g(x)-y)^2 }. \label{eq:main_alg}
\end{align}
We call this algorithm \emph{disagreement-based regression} (\dbr) and keep the dependence on $\tau$ implicit in the notation for the solution $\fdro$.\footnote{The name stems from the literature on disagreement-based active learning~\citep{hanneke2014theory}, where a similar ``range'' computation has appeared~\citep{krishnamurthy2019active,foster2018practical,foster2020instance}. However our usage is conceptually unrelated: we use disagreement for robustness to covariate shift, while, in active learning, disagreement is used to reduce sample complexity.}
There are essentially three key ingredients.
First, we introduce the ``filter'' $W_{f,g}^\tau$ to restrict the regression problem to the set of points where the predictions of $f$ and $g$ differ considerably, which we call the \emph{disagreement region}.
This formalizes the intuition that we should only measure the square loss for $f$ on points where $|f(x) - \fbar(x)| \geq c\misspec$. 
Second is the robust optimization approach, where for each $f \in \cF$, we consider all possible choices $g \in \cF$ for filtering, which allows us to take $g$ to be $L_{\infty}$-close to $\fstar$
in the analysis. 
Finally, we measure the square loss \emph{regret} in the disagreement region, by subtracting off the square loss of the comparator function $g$. 
Similar to~\cite{agarwal2022minimax}, this accounts for the fact that each $g \in \cF$ yields a different regression problem, with potentially different Bayes error rates.\footnote{More directly, the probability mass of filtered points $\Ptrain\brk{W_{f,g}^\tau(x)}$ could vary considerably for different $f,g \in \cF$.}

As our main theorem, we show that disagreement-based regression enjoys the following guarantee.
\begin{theorem}[Main result for \dbr]
\label{thm:main}
Fix $\delta \in (0,1)$. Let $\cF$ be a function class with $|\cF| < \infty$ satisfying~\pref{assum:misspecification} and~\pref{assum:boundedness}. Then with probability at least $1-\delta$, $\fdro$ with $\tau \geq 3\misspec$ satisfies
\begin{align}
\Etrain\brk*{\ind\crl*{ |\fdro(x) - \fstar(x)|\geq \tau+\misspec } \cdot \crl*{(\fdro(x) - \fstar(x))^2 - \misspec^2 %
 } } \leq \frac{160 \log (2|\cF|/\delta)}{3n}, \label{eq:dis_risk}
\end{align}
which directly implies
\begin{align}
\Ptrain\brk*{ \abs*{ \fdro(x) - \fstar(x) } \geq \tau + \misspec } \leq \frac{ 160 \log (2 |\cF|/\delta)}{3n(\tau^2 + 2\tau\misspec)}. \label{eq:pr_risk}
\end{align}
\end{theorem}
Before turning to a discussion of~\pref{thm:main} we state two immediate corollaries. The first addresses the adversarial covariate shift setting, bounding the risk of $\fdro$ under $\Dtest$. 
\begin{restatable}[Covariate shift for \dbr]{corollary}{dbrcs}
\label{cor:dbr_cs}
Fix $\delta \in (0,1)$. Under~\pref{assum:covariate_shift}--\pref{assum:boundedness}, with probability at least $1-\delta$, $\fdro$ with $\tau = 3\misspec$ satisfies
\begin{align}
\Rtest(\fdro) \leq 17\misspec^2 + O\prn*{\Cinf \frac{\log (|\cF|/\delta)}{n}}.
\end{align}
\end{restatable}
The next result shows that $\fdro$ recovers the optimal guarantee in the well-specified case, i.e., when $\misspec=0$.
\begin{restatable}[Well-specified case]{corollary}{dbrrealizable}
\label{cor:dbr_realizable}
Fix $\delta \in (0,1)$. Under~\pref{assum:covariate_shift}--\pref{assum:boundedness} (with $\misspec = 0$), with probability at least $1-\delta$, $\fdro$ with $\tau \leq O\prn*{\sqrt{\log (|\cF|/\delta)/n}}$ satisfies
\begin{align}
\Rtrain(\fdro) \leq O\prn*{ \frac{\log (|\cF|/\delta)}{n} } \quad \mathrm{ and } \quad \Rtest(\fdro) \leq O\prn*{ \Cinf \frac{\log (|\cF|/\delta)}{n} }. \label{eq:dro_realizable}
\end{align}
\end{restatable}
We now turn to some remarks regarding~\pref{thm:main} and the corollaries.

\tightparagraph{\textnormal{\dbr} avoids misspecification amplification} Comparing~\pref{cor:dbr_cs} in the $n \to \infty$ limit with~\pref{prop:erm_lb} highlights the main qualitative difference between \dbr and \erm. 
\dbr attains $O(\misspec^2)$ asymptotic test error while the test error for \erm is lower bounded by $\Omega(\Cinf \misspec^2)$.
In other words, \dbr avoids misspecification amplification while \erm does not. 
At the same time, the statistical term is identical (up to constants) to that of \erm, enabling us to recover the optimal rate in the well-specified case.

\tightparagraph{Quantile guarantee}
Taking $\tau = O(\misspec)$ in~\pref{eq:dis_risk}, we have that $\Ptrain[ |\fdro(x) - \fstar(x)| \geq c\misspec] \lesssim \frac{1}{n\misspec^2}$, which controls the large quantiles of the prediction error. 
This is reminiscent of what can be achieved by applying Markov's inequality to the guarantee for \erm in the well-specified case. %
In contrast, \erm only ensures that $\Rtrain(\ferm) = \Omega(\misspec^2)$ under misspecification, which does not imply any meaningful quantile guarantee. 
One interpretation of our results is that, although such quantile guarantees are not possible for \erm under misspecification, there is no information-theoretic obstruction.
We also note that these quantile guarantees are rather different from sup-norm convergence; see~\pref{sec:related} for further discussion.

\tightparagraph{Computational efficiency}
\dbr, as described in~\pref{eq:main_alg}, does not appear to be computationally tractable, primarily due to the non-smoothness and non-convexity introduced by the filter $W_{f,g}$. 
A natural direction for future work is to understand the computational challenges involved in avoiding misspecification amplification.

\subsubsection{Extensions}
Before closing this section, we mention two extensions that we defer to~\pref{app:extensions}.
\begin{itemize}
\item \emph{Approximation factor.} The approximation factor of $17$ in~\pref{cor:dbr_cs} can be improved to $10$ (cf. \cref{prop:improved-approx}); however
our approach for doing so degrades the convergence rate of the statistical term. 
We do not know the optimal approximation factor for this setting or whether there is an inherent trade-off between the statistical term and the approximation/misspecification term. 
\item \emph{Adapting to unknown misspecification.}
\pref{thm:main} requires setting $\tau\geq 3\misspec$ which can always be achieved by setting $\tau$ sufficiently large.
However, setting $\tau = O(\misspec)$ yields the best guarantee, and so, we would like to choose $\tau$ in a data-dependent fashion to adapt to the misspecification level. 
\pref{prop:adapting-misspeci} shows that this can be done while recovering essentially the same guarantee as in~\pref{thm:main}.
\end{itemize}

%% file: lb_fig.tex
\begin{tikzpicture}
\draw[thick] (0,0) -- (4,0);
\node[] at (4.2,0) {$\cX$};
\draw[thick] (0,0) -- (0,2.5);
\draw[thick] (0,1) -- (4,1);
\node[] at (2.25, 0.7) {$\fstar$};
\draw[thick,blue] (0,1.3) -- (4,1.3) node[anchor = south] {$\fbar$};
\draw [thick, decorate, decoration = {brace,mirror}] (4.1,1) --  (4.1,1.3);
\node[] at (4.5,1.15) {$\misspec$};
\draw[thick,red] (0, 2.2) -- (1,2.2) -- (1,1.04) -- (4, 1.04); 
\node[red] at (1.4,2.1) {$\fbad$};
\draw [decorate,decoration = {brace,mirror},thick] (0,-0.1) --  (4,-0.1);
\node[] at (2,-0.4) {$\Dtrain$};
\draw [decorate,decoration = {brace},thick] (0,+0.1) --  (1,+0.1);
\node[] at (0.5,+0.4) {$\Dtest$};
\end{tikzpicture}

%% file: section_proof.tex
This section contains the proof of~\pref{thm:main}---which we emphasize only requires elementary arguments---and is not essential for understanding the main results of the paper. 
A reader interested in applications of~\pref{thm:main} to reinforcement learning can proceed to~\pref{sec:rl}. 

The proof of~\pref{thm:main} is organized into three steps, each of which is fairly simple. 
It is helpful to define empirical and population versions of the pairwise objective used by \dbr:
\begin{align*}
\mathrm{(Empirical):}~~\cLhat(f;g) & := \frac{1}{n}\sum_{i=1}^n W_{f,g}^\tau(x_i) \crl*{(f(x_i)-y_i)^2 -(g(x_i)-y_i)^2}, \\
\mathrm{(Population):}~~\cL(f;g) &:= \Etrain\brk*{ W_{f,g}^\tau(x) \crl*{(f(x)-y)^2 -(g(x)-y)^2}}.
\end{align*}
First, we establish a certain non-negativity property of the population objective, which is the main structural result.
The second step is a uniform convergence argument to show that $\cLhat(\cdot;\cdot)$, which appears in the algorithm, concentrates to the population counterpart $\cL(\cdot;\cdot)$. 
Finally, we study the minimizer $\fdro$ 
and an $L_{\infty}$-approximation $\fbar$ and relate their objective values to establish the theorem. 
Details and proofs for the corollaries are deferred to~\pref{app:dbr}.

\paragraph{Step 1: Non-negativity}
The key lemma for the analysis is the following structural property. 
\begin{restatable}[Non-negativity]{lemma}{nonnegative}
\label{lem:nonnegative}
With $\tau \geq 2\misspec$ and for any $\fbar\in\cF$ such that $\|\fbar - \fstar\|_{\infty} \leq \misspec$, we have
\begin{align*}
\cL(f; \fbar) \geq (\tau^2 - 2\tau\misspec) \Pr\brk{ W_{f,\fbar}^\tau(x) } \geq 0.
\end{align*}
\end{restatable}
The proof requires only algebraic manipulations and actually reveals a stronger property: with $\tau \geq 2\misspec$, the random variable $W_{f,\fbar}^\tau(x)\brk*{(f(x) - \fstar(x))^2 - (\fbar(x) - \fstar(x))^2}$ is non-negative almost surely.
By the symmetry $\cL(f;g) = - \cL(g; f)$, the lemma also shows that any $L_{\infty}$-misspecified $\fbar$ has non-positive population objective.

\paragraph{Step 2: Uniform convergence}
Next we establish the following concentration guarantee.
\begin{restatable}[Concentration]{lemma}{concentration}
\label{lem:concentration}
Fix $\delta \in (0,1)$ and $\tau \geq 3\misspec$ and define $\epsstat := \frac{80 \log (|\cF|/\delta)}{3n}$. Under~\pref{assum:misspecification}, for any $\fbar \in \cF$ such that $\|\fbar - \fstar\|_{\infty} \leq \misspec$, with probability at least $1-\delta$ we have
\begin{align*}
\forall f \in \cF: ~~ %
\cL(f;\fbar) \leq 2 \cLhat(f;\fbar) + \epsstat, \quad \textrm{ and equivalently, } \quad \cLhat(\fbar; f) \leq \frac{1}{2} \prn*{ \cL(\fbar; f) + \epsstat}.
\end{align*}
\end{restatable}
The proof is based on Bernstein's inequality and importantly exploits a ``self-bounding'' property of $\cLhat(f;g)$---in particular that $\Var\brk{\cLhat(f;\fbar)} \leq \prn{\nicefrac{12}{n}}\cL(f;\fbar)$---analogously to the analysis for \erm in the well-specified case. 

\paragraph{Step 3: Analysis of $\fdro$}
Let $\fbar \in \cF$ be any function that is $L_{\infty}$-close to $\fstar$ and condition on the high probability event in~\pref{lem:concentration} holding with the choice $\fbar$. 
The \dbr minimizer satisfies
\begin{align*}
\cL(\fdro;\fbar) & \leqlab{\mathrm{(i)}} 2 \cLhat(\fdro;\fbar) + \epsstat \leqlab{\mathrm{(ii)}} 2 \max_{g\in\cF} \cLhat(\fdro;g) + \epsstat\\
& \leqlab{\mathrm{(iii)}} 2 \max_{g \in \cF} \cLhat(\fbar;g) + \epsstat %
\leqlab{\mathrm{(iv)}} \max_{g \in \cF} \cL(\fbar;g) + 2\epsstat \leqlab{\mathrm{(v)}} 2\epsstat.
\end{align*}
Here inequalities $\mathrm{(i)}$ and $\mathrm{(iv)}$ are applications of~\pref{lem:concentration}, $\mathrm{(ii)}$ and $\mathrm{(iii)}$ follow from the definition of $\fdro$ since $\fbar \in \cF$, and $\mathrm{(v)}$ is an application of~\pref{lem:nonnegative} along with the symmetry $\cL(f;g) = - \cL(g;f)$. 
\pref{eq:dis_risk} now follows from the fact that 
$W_{f,\fbar}^\tau(x) \geq \ind\crl{ |f(x) - \fstar(x)| \geq \tau+\misspec}$. 
\pref{eq:pr_risk} follows since under the event $|f(x) - \fstar(x)| \geq \tau + \misspec$ we can lower bound $(f(x) -\fstar(x))^2 - (\fbar(x) - \fstar(x))^2 \geq (\tau+\misspec)^2  - \misspec^2$.

%% file: section_rl.tex
In this section, we deploy disagreement-based regression to obtain new results in offline and online RL with function approximation. 
Algorithmically, this is achieved by using \dbr as a drop-in replacement for square loss regression in existing algorithms. We illustrate this by examining and improving the Bellman residual minimization (a.k.a. minimax) algorithm for offline RL ~\citep{antos2008learning,chen2019information} (\cref{sec:offline-rl}) and the \golf algorithm~\citep{jin2021bellman} for online RL (\cref{sec:online-rl}).
The analyses also require minimal modifications to those of~\citet{xie2021batch} and~\citet{xie2022role}, respectively. 
To emphasize the ease with which \dbr can be applied, we adopt the formulations and much of the notation from these works.
All proofs for results in this section are deferred to~\pref{app:rl}.

\subsection{Offline reinforcement learning}\label{sec:offline-rl}
\paragraph{Setup and notation}
We consider a discounted Markov decision process (MDP) $M = (P, R, d_0, \gamma)$ over states $\cS$ and actions $\cA$, where $P: \cS \times \cA \to \Delta(\cS)$ is the transition operator, $R: \cS \times \cA \to [0,1]$ is the reward function, $d_0 \in \Delta(\cS)$ is the initial state distribution, and $\gamma \in [0,1)$ is the discount factor. 
A policy $\pi: \cS \to \Delta(\cA)$ induces a trajectory $s_0,a_0,r_0,s_1,a_1,r_1,\ldots$ where $s_0 \sim d_0$, and for each $h \in \bbN$, $a_h \sim \pi(s_h)$, $r_h = R(s_h,a_h)$, and $s_{h+1} \sim P(s_h,a_h)$. 
We use $\bbP^{\pi}[\cdot]$ and $\bbE^{\pi}[\cdot]$ to denote probability and expectation under this process. 
Let $d_h^\pi \in \Delta(\cS\times\cA)$ denote the occupancy measure of $\pi$ at time-step $h$, defined as $d_h^\pi(s,a) := \bbP^{\pi}[ s_h = s, a_h=a]$ and let $d^\pi := (1-\gamma) \sum_{h=0}^{\infty} \gamma^h d_h^\pi$. 

The value of $\pi$ is denoted $J(\pi) := \bbE^{\pi}\brk*{\sum_{h=0}^{\infty} \gamma^h r_h }$. Each policy $\pi$ has value functions $V^\pi: s \mapsto \bbE^{\pi}\brk*{\sum_{h=0}^{\infty} \gamma^h r_h  \mid s_0 = s}$ and $Q^\pi : (s,a) \mapsto \bbE^{\pi}\brk*{\sum_{h=0}^{\infty} \gamma^h r_h  \mid s_0 = s, a_0 = a}$, and it is known that there exists a policy $\pi^\star$ that maximizes $V^\pi(s)$ simultaneously for all $s \in \cS$. 
This policy also optimizes $J(\cdot)$ and hence is called the optimal policy. 
It is also known that the value function $Q^\star := Q^{\pi^\star}$ induces the optimal policy via $\pi^\star: s\mapsto \argmax_{a} Q^\star(s,a)$ and additionally satisfies \emph{Bellman's optimality equation}: $Q^\star(s,a) := \brk*{\cT Q^\star}(s,a)$ where $\cT$ is the Bellman operator, defined via $\cT f : (s,a) \mapsto \bbE\brk*{ r_0 + \gamma \max_{a'} f(s_1,a') \mid s_0=s,a_0=a}$. 

In the offline value function approximation setting, we are given a dataset of $n$ tuples $D_n \coloneqq \{(s_i,a_i,r_i,s_i')\}_{i=1}^n$ generated i.i.d. from the following process: $(s_i,a_i) \sim \mu$ where $\mu \in \Delta(\cS,\cA)$ is the \emph{data collection distribution}, $r_i = R(s_i,a_i)$, and $s_i' \sim P(s_i,a_i)$.
We are also given a function class $\cF \subset (\cS \times \cA \to \bbR)$, where each $f \in \cF$ induces the policy $\pi_f: s \mapsto \argmax_{a} f(s,a)$. 
Given dataset $D_n$ and function class $\cF$, we seek a policy $\hat{\pi}$ that has small suboptimality gap: $J(\pi^\star) - J(\hat{\pi})$.
We impose the following assumptions on the function class and on the data collection distribution:
\begin{itemize}
\item \textbf{$L_{\infty}$-misspecified realizability/completeness}: There exists $\fbar \in \cF$ such that  
$\| \fbar - \cT\fbar \|_{\infty} \leq \misspec$. 
Additionally, for any $f \in \cF$ there exists $g \in \cF$ such that 
$ \| g - \cT f\|_{\infty} \leq \misspec$. 
\item \textbf{Concentrability}: There exists a constant $\Cconc\in[1,\infty)$ such that $\max_{\pi \in \Pi} \nrm*{ \frac{d^\pi}{\mu}}_{\infty} \leq \Cconc$. Here $\Pi := \{\pi_f: f \in \cF\}$ is the policy class induced by $\cF$. 
\end{itemize}
There is a large body of recent work studying various function approximation and coverage assumptions in offline RL~\citep[c.f.,][]{xie2021batch}.
Arguably the most standard are concentrability, as we use, and \emph{exact} realizability/completeness, which is stronger than our version with misspecification.
Regarding the function approximation assumption, it is not hard to show that misspecification amplification---which in this setting is defined by the suboptimality $J(\pi^\star) - J(\hat{\pi})$ scaling as $\Omega(\misspec\sqrt{\Cconc})$---is necessary under weaker notions, such as $L_2(\mu)$-misspecification. 
Regarding coverage, as we will discuss below, the strength of the coverage assumption determines whether misspecification amplification can be avoided or not.

\paragraph{Algorithm and guarantee}
The algorithm we study is a minor modification to the minimax algorithm~\citep{antos2008learning,chen2019information}. 
For each function $\tilde{f} \in \cF$ and each tuple $(s_i,a_i,r_i,s_i')$ we can form a regression sample $(s_i,a_i,y_{\tilde{f},i} := r_i + \gamma \max_{a'} \tilde{f}(s_i',a'))$ and define the predictor $\fhat$ via the objective:
\begin{equation}\label{eq:drmbrm}
\fhat := \argmin_{f \in \cF} \max_{g \in \cF} \frac{1}{n}\sum_{i=1}^n W_{f,g}^\tau(s_i,a_i) \crl*{ (f(s_i,a_i) - y_{f,i})^2 - (g(s_i,a_i) - y_{f,i})^2 }.
\end{equation}
Here $W_{f,g}^\tau(\cdot)$ is the filter in~\pref{eq:main_alg} with $x = (s,a)$. 
Given $\fhat$, we output $\hat{\pi} := \pi_{\fhat}$.
Note that the only difference between this algorithm and the original minimax algorithm is the use of the filter $W_{f,g}^\tau(\cdot)$ which is essential for obtaining the following guarantee. 

\begin{restatable}[\dbr for offline RL]{theorem}{offlinerl}
\label{thm:offline_rl}
Fix $\delta \in (0,1)$, assume that $\cF$ is $L_{\infty}$-misspecified and $\mu$ satisfies concentrability (as defined above). Consider the algorithm defined in Eq. \eqref{eq:drmbrm} with $\tau = 3\misspec$.  Then, 
with probability at least $1-\delta$ we have
\begin{align*}
J(\pi^\star) - J(\hat{\pi}) \leq O\prn*{ \frac{\misspec}{1-\gamma} + \frac{1}{1-\gamma}\sqrt{\Cconc\frac{\log(|\cF|/\delta) }{n}}}.
\end{align*}
\end{restatable}
The theorem is best understood via comparison to the guarantee for the standard minimax algorithm, e.g., Theorem 5 of~\citet{xie2020q}. 
Under our assumptions ($L_{\infty}$-misspecification and concentrability), these two bounds differ \emph{only} in the misspecification term: our theorem scales as $\misspec/(1-\gamma)$ while the guarantee for the minimax algorithm scales as $\misspec\sqrt{\Cconc}/(1-\gamma)$.\footnote{\citet{xie2020q} consider slightly weaker assumptions: they measure both misspecification and concentrability via the $L_2(\mu)$ norm. Our analysis easily accommodates $L_2(\mu)$-concentrability, as can be seen from the proof. On the other hand, as described in \cref{sec:erm-amplification}, misspecification amplification is necessary under $L_2(\mu)$-misspecification.}
Thus, our algorithm inherits the favorable properties of \dbr to avoid misspecification amplification in offline RL.

This feature is notable in light of existing lower bounds for misspecified RL~\citep{du2019good,van2019comments,lattimore2020learning}. %
Formally, these results consider linear function approximation in various online RL models, but the constructions can be extended to offline RL with general function approximation where coverage is measured via the \emph{Bellman transfer coefficient}.
This coefficient is the smallest $\Ctrans$ such that $\max_{\pi, f \in \cF} \frac{\nrm*{ f - \apx[f]}^2_{L_2(d^\pi)}}{\nrm*{ f - \apx[f]}^2_{L_2(\mu)}} \leq \Ctrans$ where $\apx[f] \in \cF$ is the $L_{\infty}$-approximation of $\cT f$.\footnote{Many Bellman transfer coefficients exist, but a standard one is the smallest $\Ctrans$ such that $\max_{\pi, f \in \cF} \nicefrac{\nrm*{ f - \cT f}^2_{L_2(d^\pi)}}{\nrm*{ f - \cT f}^2_{L_2(\mu)}} \leq \Ctrans$. This coincides with ours under exact realizability/completeness, but we believe our definition is more appropriate for the misspecified case because it is equivalent to feature coverage under linear function approximation. Indeed, if $\cF$ consists of linear functions in some feature map $\phi: \cS\times \cA \to \bbR^d$ (but $\cT f$ may not be linear due to misspecification) then our definition can be expressed via the features (as $\max_{\pi,\theta\in \bbR^d} \nicefrac{ \theta^\top \Sigma_\pi \theta}{\theta^\top \Sigma_{\mu} \theta}$ where $\Sigma_d = \bbE_{d} \brk*{ \phi(s,a) \phi(s,a)^\top }$) but the standard definition cannot.}
The lower bound states that an asymptotic error of $\Omega(\misspec\sqrt{\Ctrans})$ is unavoidable. 

To contextualize our result with this lower bound, we identify two regimes: the ``Bellman transfer regime'' where $\Ctrans < \infty$ and the ``concentrability regime'' where $\Cconc < \infty$, and note that, since $\Ctrans \leq \Cconc$, the former is more general. 
In the Bellman transfer regime, misspecification amplification is unavoidable. 
In the concentrability regime,~\pref{thm:offline_rl} avoids misspecification amplification \emph{and} is sample efficient (i.e., has statistical term scaling as $\poly(\Cconc, \log (|\cF|/\delta), \frac{1}{n}, \frac{1}{1-\gamma})$). 
This is the first result showing that both of these properties are simultaneously achievable: prior results achieve sample efficiency with misspecification amplification~\citep[e.g.,][]{xie2020q}, or avoid misspecification amplification with undesirable sample complexity scaling as $\poly(|\cS|)$ (the latter is easily achieved under concentrability via a tabular model-based approach). %
Thus, the regime determines whether misspecification amplification is avoidable or not, and, in the regime where it is avoidable, our algorithm does so in a sample-efficient manner. %

\subsection{Online reinforcement learning}\label{sec:online-rl}
\paragraph{Setup and notation} We consider a finite horizon episodic MDP $(P,R,H,s_1)$ over state space $\cS$ and action space $\cA$, where $H \in \bbN$ is horizon, $P := \{P_h\}_{h=1}^H$ with $P_h : \cS \times \cA \to \Delta(\cS)$ is the non-stationary transition operator, $R:= \{R_h\}_{h=1}^H$ with $R_h: \cS \times \cA \to [0,1]$ is the non-stationary reward function, and $s_1$ is a fixed starting state. 
A (non-stationary) policy $\pi := \{\pi_h\}_{h=1}^H$ is a sequence of mappings $\pi_h: \cS \to \Delta(\cA)$ which induces a trajectory $(s_1,a_1,r_1,\ldots,s_H,a_H,r_H)$ where $a_h \sim \pi_h(s_h), r_h = R_h(s_h,a_h)$ and $s_{h+1} \sim P_h(s_h,a_h)$ for each time step. 
We use $\bbP^{\pi}[\cdot]$ and $\bbE^{\pi}[\cdot]$ to denote probability and expectation under this process, respectively. 
Let $d_h^\pi \in \Delta(\cS\times\cA)$ denote the occupancy measure of $\pi$ at time-step $h$, defined as $d_h^\pi(s,a) := \bbP^{\pi}[ s_h = s, a_h=a]$.

The value of policy $\pi$ is denoted $J(\pi) := \bbE^{\pi}\brk*{\sum_{h=1}^H r_h}$. 
Each policy has value functions: $V_h^\pi: s \mapsto \bbE^{\pi} \brk*{ \sum_{h' = h}^H r_{h'} \mid s_h = s}$ and $Q_h^\pi: (s,a) \mapsto \bbE^{\pi}\brk*{ \sum_{h'=h}^H r_{h'} \mid s_h = s, a_h = a}$ and there exist an optimal policy $\pi^\star = \{\pi_h^\star\}_{h=1}^H$ that maximizes $V_h^\pi$ simultaneously for each state $s \in \cS$ and hence maximizes $J(\cdot)$. 
The optimal value function $Q_h^\star := Q_h^{\pi_h^\star}$ induces $\pi^\star$ via $\pi_h^\star: s \mapsto \argmax_{a} Q_h^\star(s,a)$ and satisfies Bellman's equation: $Q_h^\star(s,a) = [\cT_hQ_{h+1}^\star](s,a)$ where the Bellman operator $\cT_h$ is defined via $[\cT_h f_{h+1}] (s,a) = R_h(s,a) + \bbE\brk*{\max_{a'} f_{h+1}(s_{h+1},a')\mid s_h = s, a_h = a}$. We assume per-episode rewards satisfy $\sum_{h=1}^H r_h \in [0,1]$. %

In online RL, we interact with the MDP for $T$ episodes, where in each episode we select a policy $\pit$ and collect the trajectory $(s_1\index{t}, a_1\index{t},r_1\index{t},\ldots,s_H\index{t}, a_H\index{t}, r_H\index{t})$ by taking actions $a_h\index{t} = \piht(s_h\index{t})$. We measure performance via the cumulative regret, define as $\reg := \sum_{t=1}^T J(\pi^\star ) - J(\pit)$. 
We equip the learner with a value function class $\cF := \cF_1\times\ldots\times \cF_H$ where each $\cF_h \subset \cS \times \cA \to [0,1]$.
Each $f \in \cF$ induces a policy $\pi_f$ which, at time step $h$ takes actions via $\pi_{f,h}(s_h) = \argmax_a f_h(s_h,a_h)$. 
We make the following assumptions:
\begin{itemize}
\item \textbf{$L_{\infty}$-approximate realizability/completeness.} For each $h\in [H]$ there exists $\fbar_h \in \cF_h$ such that $\nrm*{ \fbar_h - \cT_h\fbar_{h+1}}_{\infty} \leq \misspec$.%
Additionally, for each $f_{h+1} \in \cF_{h+1}$ there exists $f_h \in \cF_h$ such that $\nrm*{ f_h - \cT_{h}f_{h+1}}_{\infty} \leq \misspec$. 
\item \textbf{Coverability.} There exists a constant $\Ccov\in[1,\infty)$ such that $\inf_{\mu_1,\ldots,\mu_H \in \Delta(\cS\times\cA)} \sup_{\pi\in\Pi, h} \nrm*{ \frac{d_h^\pi}{\mu_h}}_{\infty} \leq \Ccov$. Here $\Pi := \{ \pi_f: \pi_{f,h}(s) = \argmax_{a} f_h(s,a), f \in \cF\}$ is the policy class induced by $\cF$. 
\end{itemize}
As in offline RL, there is a large body of recent work studying function approximation and structural conditions for sample-efficient online RL~\citep[c.f.,][]{agarwal2019reinforcement,foster2023foundations}.
It is fairly standard to assume exact realizability and completeness, which is stronger than our version with misspecification.
Coverability is a recently proposed structural condition~\citep{xie2022role}: $\Ccov$ is known to be small in many MDP models of interest, but weaker conditions that enable sample-efficiency are known. 
As we will see, the strength of the structural condition determines whether misspecification amplification can be avoided or not. 

\paragraph{Algorithm and guarantee}
The algorithm is a very minor modification to \golf~\citep{jin2021bellman,xie2022role}. 
To condense the notation, given a sample $(s_h\index{i},a_h\index{i}, r_h\index{i}, s_{h+1}\index{i})$ and a function $f' \in \cF_{h+1}$, define $x_h\index{i} := (s_h\index{i}, a_h\index{i})$ and $y\index{i}_{f',h} := r_h\index{i} + \max_{a'} f'(s_{h+1}\index{i}, a')$. 
At the beginning of episode $t$, define a version space
\begin{align*}
\cF\index{t-1} &:= \crl*{ f \in \cF: \forall h \in [H]: \max_{g_h \in \cF_h} \sum_{i=1}^{t-1} W_{f_h,g_h}^{\tau}(x_h\index{i})\crl*{(f_h(x_h\index{i}) - y_{f_{h+1},h}\index{i})^2-(g_h(x_h\index{i}) - y_{f_{h+1},h}\index{i})^2 } \leq \beta},
\end{align*}
where $\beta>0$ is a hyperparameter we will set below. 

Then, we define the optimistic value function $f\index{t} := \argmax_{f \in \cF\index{t-1}} f_1(s_1, \pi_{f,1}(s_1))$ and the induced policy
$\pi\index{t} := \pi_{f\index{t}}$, collect a trajectory via $\pi\index{t}$, and proceed to the next episode. 
Note that the only difference between this algorithm, which we call \golfdbr, and the version of \golf studied by~\citet{xie2022role} is that we use the filter $W_{f_h,g_h}^{\tau}(\cdot)$ in the construction of the version space. \golfdbr enjoys the following guarantee.
\begin{restatable}[\dbr for online RL]{theorem}{onlinerl}
\label{thm:online_rl}
Fix $\delta \in (0,1)$, and assume that $\cF$ is $L_\infty$-misspecified and $\mu$ satisfies coverability (as defined above). Consider \golfdbr with $\tau = 3\misspec$ and $\beta = c \log(TH|\cF|/\delta)$. Then, with probability at least $1-\delta$, we have
\begin{align*}
\reg \leq O\prn*{ \misspec HT + H\sqrt{\Ccov T\log(TH|\cF|/\delta)\log(T)}}.
\end{align*}
\end{restatable}
Paralleling the discussion following~\pref{thm:offline_rl}, we emphasize two aspects of the result. 
The first is that it extends Theorem 1 of~\citet{xie2022role} to the misspecified setting, with no degradation of the statistical term and without incurring a dependence on $\misspec\sqrt{\Ccov}$.
In other words, it avoids misspecification amplification. 

The second remark is that, when taken with existing lower bounds~\citep{du2019good,van2019comments,lattimore2020learning},~\pref{thm:online_rl} establishes a separation between coverability and structural parameters defined in terms of Bellman errors, which include the Bellman-Eluder dimension~\citep{jin2021bellman}, bilinear rank~\citep{du2021bilinear}, and Bellman rank~\citep{jiang2017contextual}.\footnote{As with Bellman transfer coefficients, we believe these definitions should be adjusted to accommodate misspecification. See Definition 10 in~\citet{jiang2017contextual} for an example.}
This separation is more subtle than in offline RL, because here, as long as the state-action space is finite, one can always use a ``tabular'' method and eliminate misspecification altogether, at the cost of $\poly(|\cS|,|\cA|)\cdot\sqrt{T}$ regret. 
To rule out this algorithm, we restrict to sample-efficient methods: in a setting where a particular structural parameter (e.g., coverability or Bellman rank) is bounded by $d$ we say that an algorithm is sample-efficient if its statistical term scales as $\poly(d, \log(|\cF|/\delta),H) \cdot o(T)$. 
The lower bounds show that, when the structural parameter involves Bellman errors (like the Bellman rank), $\misspec T\sqrt{d}$ misspecification error is necessary for sample efficient algorithms.\footnote{Formally, for any $\zeta > 0$ one requires at least $\exp(d^{2\zeta})$ samples to find a $d^{1/2-\zeta}\misspec$ suboptimal policy~\citep{lattimore2020learning}. }
On the other hand, under coverability, we can achieve misspecification error with no dependence on the structural parameter, in a sample efficient manner.\footnote{We believe that misspecification error $\misspec HT$ is optimal under coverability and that $\misspec HT\sqrt{d}$ is optimal under structural parameters like Bellman rank. However, it remains open to establish the necessity of the horizon factors.} 
This establishes that whether misspecification amplification can be avoided sample-efficiently depends on the structural properties of the MDP.
To our knowledge, this is a novel insight into the interaction between the structural and function approximation assumptions in online RL.

%% file: section_related.tex
There is a vast body of work studying distribution shift broadly and covariate shift in particular. 
We focus on the most closely related techniques for the covariate shift setting and refer the reader to~\cite{quinonero2008dataset,sugiyama2012machine,shen2021towards} for a more comprehensive treatment.

\textbf{Reweighting and robust optimization.}
Perhaps the most common way to correct for covariate shift is by reweighting each example $(x,y)$ in the objective function by the density ratio $w(x) := \dtest(x)/\dtrain(x)$. 
This method has been studied in a long series of works~\citep{shimodaira2000improving,cortes2010learning,cortes2014domain}.
In its simplest form it requires knowledge of $\Dtest$ via the density ratios, so it is not directly applicable to our adversarial covariate shift setting.
Extensions include approaches that estimate density ratios using unlabeled samples from $\Dtest$~\citep{huang2006correcting,sugiyama2007direct,gretton2009covariate,yu2012analysis} and robust optimization approaches that employ an auxiliary hypothesis class of distributions $\cP$ containing $\Dtest$~\citep{hashimoto2018fairness,sagawa2019distributionally,duchi2021learning,agarwal2022minimax}. 
However, these still require prior knowledge about $\Dtest$, in particular it is known that the sample complexity of robust optimization scales with the statistical complexity of the auxiliary class $\cP$~\citep{duchi2021learning}, leading to vacuous bounds in the absence of inductive bias. 

\citet{ge2023maximum} study statistical inference under covariate shift in well- and misspecified settings.
They show that maximum likelihood estimation on $\Dtrain$ is inconsistent with misspecification, a result which is conceptually similar to our lower bound for \erm.
However, their construction is not $L_{\infty}$-misspecified so it is not directly comparable. 
Algorithmically, they use reweighting for the misspecified case, which, as mentioned, cannot be implemented in our setting. 

\textbf{Sup-norm convergence and function class-specific results.}
Another line of work provides specialized analyses for specific function classes of interest, such as linear~\citep{lei2021near}, nonparametric~\citep{kpotufe2018marginal,pathak2022new,ma2023optimally}, and some neural network~\citep{dong2022first} classes.
The overarching technical approach in these works is to measure distance between distributions in manner that captures the structure of the function class, analogously to learning-theoretic results for domain adaptation~\citep{ben2006analysis,mansour2009domain}.
A complementary approach is based on sup-norm convergence which seeks to control $\|\fhat - \fstar\|_{\infty}$ for a predictor $\fhat$
and is naturally robust to covariate shift. 
Sup-norm convergence has been studied for various function classes~\citep[c.f.,][]{schmidt2022local,dong2023toward},
but unfortunately is not possible in the general statistical learning setup~\citep{dong2023toward}. 
We mention sup-norm convergence primarily to contrast with our quantile guarantee in~\pref{eq:pr_risk}, which controls the probability over $x$ of large errors rather than the magnitude of the errors themselves and which is attainable for any function class, even with misspecification. 
All of these works differ from ours in that (a) they consider specific function classes and (b) they operate closer to the well-specified regime than we do (e.g., in the nonparametric setting, one can drive the 
misspecification error to zero).

\textbf{Related work in reinforcement learning.}
Our results for offline and online RL build directly on the analyses in~\citet{xie2020q} and~\citet{xie2022role} respectively. 
The former contributes to a long line of work on offline RL~\citep{munos2003error,munos2007performance,antos2008learning,chen2019information} while the latter is part of a series of works establishing structural conditions under which online reinforcement learning is statistically tractable~\citep[c.f.,][]{agarwal2019reinforcement,foster2023foundations}. %
Many of these works do account for misspecification, but the question of whether misspecification amplification can be avoided is not considered. 

Results that do focus on misspecification primarily consider linear function approximation. 
In the simpler offline policy evaluation setting, several works study least squares temporal difference learning (LSTD)~\citep{bradtke1996linear} with misspecification~\citep{tsitsiklis1996analysis,yu2010error,mou2020optimal}.
Recently,~\citet{amortila2023optimal} precisely characterized the optimal misspecification amplification (i.e., approximation factors) achievable across a range of settings, showing that LSTD is essentially optimal in most regimes.
The exception is when the offline data distribution is supported on the entire state space, one can employ a ``tabular'' model-based algorithm to incur no approximation error whatsoever, but the sample complexity scales polynomially with $|\cS|$. 
Our offline RL results are conceptually similar because under concentrability (which essentially implies full support), the standard minimax algorithm does not achieve the optimal approximation factor. 
A crucial difference is that our disagreement-based variant achieves an improved approximation factor \emph{without} incurring any sample complexity overhead.

For the more challenging offline policy optimization and online RL,~\citet{du2019good,van2019comments,lattimore2020learning} establish conditions under which misspecification amplification is necessary. 
As discussed above, combining our results with these lower bounds and their variations, reveals new tradeoffs between coverage/structural and function approximation conditions, distinct from tradeoffs established by prior work~\citep{xie2021batch,foster2021offline}.

%% file: section_discussion.tex
This paper highlights an intriguing interplay between misspecification and distribution shift, exposing the undesirable \emph{misspecification amplification} property of \erm, and proposing disagreement-based regression as a remedy.
We have shown that using disagreement-based regression in online and offline reinforcement learning yields new technical results and reveals new tradeoffs between coverage/structural assumptions and function approximation assumptions. 

We close by mentioning several interesting avenues for future work.
There are a number of directions that pertain to the core setting of misspecified regression under covariate shift; for example, (a) extending the analysis of \dbr to infinite function classes, other loss functions, and other notions of misspecification, (b) deriving a more computationally efficient procedure---perhaps in an oracle model of computation---that avoids misspecification amplification, and (c) determining the optimal achievable approximation factor.
Pertaining to reinforcement learning theory, we believe the most pressing direction is to deepen our understanding of the relationship between coverage/structural assumptions (for offline/online RL, respectively) and function approximation assumptions, and we believe misspecification provides a novel lens to study this relationship. 
It is also worthwhile to consider other applications involving distribution shift where \dbr or related procedures may reveal new conceptual insights. 
Finally, it would also be interesting to study empirical issues, to understand how pervasive and problematic misspecification amplification is, develop practical interventions, and consider applying them to distribution shift and deep reinforcement learning scenarios. 

In short, there is much more to understand about the interplay between misspecification and distribution shift, and we look forward to progress in the years to come.

%% file: appendix_regression.tex
\subsection{Analysis for ERM}
\ermub*
\begin{proof}[Proof of~\pref{prop:erm_ub}]
The proof of~\pref{prop:erm_ub} is fairly standard, particularly in the well-specified case when $\misspec=0$. 
Our analysis that handles misspecification is adapted from the proof of Lemma 16 in~\citet{chen2019information}. 
For the majority of the proof we only consider $\Dtrain$, and we consequently omit the subscript when indexing expectations, variances, and the risk functional. Define
\begin{align*}
R(f) := \bbE[ (f(x) - \fstar(x))^2 ] \quad \mathrm{and} \quad \Rhat(f) := \frac{1}{n}\sum_{i=1}^n (f(x_i) - y_i)^2, 
\end{align*}
so that $\ferm := \argmin_{f \in \cF} \Rhat(f)$.
We establish concentration on the ``excess risk'' functional $\Rhat(f) - \Rhat(\fbar)$. For any $f \in \cF$, we establish the following facts:
\begin{align}
\bbE\brk{ (f(x) -y)^2 - (\fbar(x) - y)^2 } &= \bbE\brk{ (f(x) - \fstar(x))^2 - (\fbar(x) - \fstar(x))^2 } \label{eq:sq_loss_mean}\\
\Var\brk{ (f(x) - y)^2 - (\fbar(x)-y)^2 } &\leq 8 \bbE\brk{ (f(x) - y)^2 - (\fbar(x)-y)^2 }  + 16 \misspec^2 \label{eq:sq_loss_var}. 
\end{align}
\pref{eq:sq_loss_mean} implies that $\bbE\brk{ \Rhat(f) - \Rhat(\fbar) }  = R(f) - R(\fbar)$ as desired. \pref{eq:sq_loss_var} will enable us to achieve a fast convergence rate. 
The former is derived as follows. Observe that conditional on any $x$ we have
\begin{align*}
& \bbE\brk{ (f(x) - y)^2 - (\fbar(x) - y)^2 \mid x} \\
& ~~~~~~ = \bbE\brk{ (f(x) - y)^2 - (\fbar(x) - \fstar(x) + \fstar(x) - y)^2 \mid x}\\
& ~~~~~~ =  \bbE\brk{ (f(x) - y)^2 - (\fbar(x) - \fstar(x))^2 - 2(\fbar(x) - \fstar(x))(\fstar(x) - y) - (\fstar(x) - y)^2 \mid x}\\
& ~~~~~~ =  \bbE\brk{ (f(x) - y)^2 - (\fbar(x) - \fstar(x))^2 - (\fstar(x) - y)^2 \mid x}\\
& ~~~~~~ = f(x)^2 - \fstar(x)^2  - 2\Etrain[y \mid x] (f(x) - \fstar(x)) - (\fbar(x) - \fstar(x))^2  \\
& ~~~~~~ = (f(x) - \fstar(x))^2 -  (\fbar(x) - \fstar(x))^2.
\end{align*}
\pref{eq:sq_loss_var} is derived as follows.
\begin{align*}
\Var\brk{ (f(x) - y)^2 - (\fbar(x)-y)^2 } &\leq \bbE\brk{ \prn{ (f(x) - y)^2 - (\fbar(x)-y)^2 }^2 } \\
& = \bbE\brk{ (f(x) - \fbar(x))^2 ( f(x) +\fbar(x) - 2y)^2 } \\
& \leq 4 \bbE \brk{ (f(x) - \fbar(x))^2 }\\
& \leq 8 \bbE \brk{ (f(x) - \fstar(x))^2 + (\fbar(x) - \fstar(x))^2 }\\
& = 8 \bbE \brk{ (f(x) - \fstar(x))^2 - (\fbar(x) - \fstar(x))^2 + 2(\fbar(x) - \fstar(x)^2}\\
& \leq 8 \bbE \brk{ (f(x) - \fstar(x))^2 - (\fbar(x) - \fstar(x))^2} + 16\misspec^2.
\end{align*}
Finally, we apply~\pref{eq:sq_loss_mean}. 

Now, Bernstein's inequality and a union bound over $f \in \cF$ gives that with probability at least $1-\delta$
\begin{align*}
\forall f \in \cF: R(f) - R(\fbar) - \prn{ \Rhat(f) - \Rhat(\fbar)}   \leq \sqrt{ \frac{\prn{16 (R(f) - R(\fbar)) + 32\misspec^2 }  \log(|\cF| /\delta)}{n}} + \frac{ 4\log(|\cF|/\delta)}{3n}.
\end{align*}
Since $\ferm$ minimizes $\Rhat(f)$ we have that $\Rhat(\ferm) - \Rhat(\fbar) \leq 0$, we can deduce that
\begin{align*}
R(\ferm) - R(\fbar) \leq \sqrt{ \frac{\prn{16 (R(\ferm) - R(\fbar)) + 32\misspec^2 }  \log(|\cF| /\delta)}{n}} + \frac{ 4\log(|\cF|/\delta)}{3n}.
\end{align*}
Using the AM-GM inequality  ($\sqrt{ab} \leq a/2 + b/2$), the right hand side can be simplified to yield
\begin{align*}
R(\ferm) - R(\fbar) \leq \frac{1}{2} \prn{ R(\ferm) - R(\fbar) } + \misspec^2 + \frac{28 \log(|\cF|/\delta)}{3n}.
\end{align*}
Re-arranging and using that $\Rtrain(\fbar) \leq \misspec^2$ we obtain
\begin{align*}
\Rtrain(\ferm) \leq 3\misspec^2 + \frac{56 \log (|\cF|/\delta)}{3n}.
\end{align*}
Finally we bound the risk under $\Dtest$ via a standard importance weighting argument:
\begin{align*}
\Rtest(\ferm) = \Etrain\brk*{ \frac{\dtest(x)}{\dtrain(x)} (\ferm(x) - \fstar(x))^2 }  \leq \sup_{x \in \cX}\abs*{\frac{\dtest(x)}{\dtrain(x)}} \cdot \prn*{ 3\misspec^2 + \frac{56 \log (|\cF|/\delta)}{3n} }. 
\end{align*}
Note that we crucially use that $(\ferm(x) - \fstar(x))^2$ is non-negative here. This proves the proposition. 
\end{proof}

\ermlb*
\begin{proof}[Proof of~\pref{prop:erm_lb}]
Fix $\misspec \in (0,1)$ and $\Cinf \geq 1$ such that $\sqrt{\Cinf}\cdot\misspec \leq 1/2$. Let $0 <\zeta < \sqrt{\Cinf}\cdot\misspec$. Let $\cX = [0,1]$ and let $\Dtrain$ be the distribution over $(x,y)$ where $x \sim \mathrm{Uniform}(\cX)$ and $y \sim \mathrm{Ber}(1/2)$.
Let $\widetilde{\cX} := [0, 1/\Cinf] \subset \cX$ and let $\Dtest$ be the distribution over $(x,y)$ where $x \sim \mathrm{Uniform}(\widetilde{\cX})$ and $y \sim \mathrm{Ber}(1/2)$. 
These choices yield $\fstar(x) = 1/2$ for all $x \in \cX$, satisfy~\pref{assum:covariate_shift}, and ensure that $\sup_{x \in \cX}\abs*{\frac{\dtest(x)}{\dtrain(x)}} = \Cinf$.  

Let $\cF = \{\fbar,\fbad\}$ where $\fbar(x) = 1/2+\misspec$ for all $x \in \cX$ (satisfying~\pref{assum:misspecification}) and $\fbad$ is defined as
\begin{align*}
\fbad(x) = \left\{ \begin{aligned}
& 1/2  & \mathrm{ if } ~ x \notin \widetilde{\cX}\\
& 1/2 + \zeta & \mathrm{ if } ~  x \in \widetilde{\cX}
\end{aligned}
\right..
\end{align*}
By definition, observe that $\ferminf = \fbad$ as long as $\nrm*{\fbad - \fstar}_{L_2(\Dtrain)}^2 < \nrm*{\fbar - \fstar}_{L_2(\Dtrain)}^2$.
A direct calculation shows that this inequality is satisfied for any $\zeta < \sqrt{\Cinf} \cdot \misspec$. 
However, $\fbad$ has large population risk under $\Dtest$, in particular
\begin{align*}
\Rtest(\fbad) = \Etest\brk{ (\fbad(x) - \fstar(x))^2} = \zeta^2,
\end{align*}
which we can make arbitrarily close to $\Cinf\misspec^2$. 
\end{proof}

\subsection{Discussion of other algorithms}
\label{app:other_algs}
\paragraph{Star algorithm} Audibert's star algorithm~\citep{audibert2007progressive,liang2015learning} is a two-stage regression procedure that achieves the fast convergence rate for non-convex classes in misspecified or agnostic regression.
Given that the construction used to prove~\pref{prop:erm_lb} has a finite (and hence non-convex) function class, one might ask whether the star algorithm can avoid misspecification amplification. 
We briefly sketch here why this is not the case. 
In the context of the construction, where $\cF = \{\fbad, \fbar\}$, the asymptotic version of the star algorithm is to compute
\begin{align*}
\fhat_{\mathrm{star}} := \argmin_{ f_\alpha: \alpha \in [0,1]} \Etrain\brk{ (f_\alpha(x) - \fstar(x))^2 } \quad \mathrm{where} \quad f_\alpha(x) = (1-\alpha)\fbad(x) + \alpha \fbar(x).
\end{align*}
We claim that when $\zeta = \sqrt{\Cinf}\cdot\misspec$, the optimal choice for $\alpha$ is exactly $1/2$. 
The prediction error under $\Dtest$ for this choice is, unfortunately, exactly $\nicefrac{1}{4}(\sqrt{\Cinf}+1)^2 \misspec^2$, which still manifests misspecification amplification. 
Note that, due to the simplicity of our construction, the same argument applies to other improper learning schemes based on convexification~\citep[c.f.,][]{lecue2014optimal}. 

To see that the minimum is achieved at $\alpha = 1/2$, we write the optimization problem over $\alpha$ as
\begin{align*}
& \argmin_{\alpha \in [0,1]} \frac{1}{\Cinf}\cdot \prn*{ (1-\alpha) \sqrt{\Cinf}\misspec + \alpha \misspec}^2 + \prn*{1 - \frac{1}{\Cinf}}\cdot\prn*{\alpha\misspec}^2
= \argmin_{\alpha \in [0,1]} \alpha^2 + (1-\alpha)^2 + \frac{2 \alpha (1-\alpha)}{\sqrt{\Cinf}}.
\end{align*}
The derivative, w.r.t. $\alpha$, of the latter is
\begin{align*}
\frac{ d\prn*{ \alpha^2 + (1-\alpha)^2 + \frac{2 \alpha (1-\alpha)}{\sqrt{\Cinf}} }}{d\alpha} = 2\alpha - 2(1-\alpha) + \frac{2}{\sqrt{\Cinf}} - \frac{4\alpha}{\sqrt{\Cinf}} = \prn*{2- \frac{2}{\sqrt{\Cinf}}}(2\alpha-1).
\end{align*}
Since $\Cinf > 1	$, the second derivative is non-negative, so the optimization problem is convex. 
Moreover, the derivative is zero at $\alpha=1/2$, showing that this is a minimizer of the optimization problem. 

\paragraph{$L_{\infty}$ regression}
Given that we assume $L_{\infty}$-misspecification, and in light of the construction for~\pref{prop:erm_lb}, it is tempting to optimize the maximal absolute deviation instead of  the square loss:
\begin{align*}
\finf \gets \argmin_{ f \in \cF} \max_{i} \abs{ f(x_i) - y_i }.
\end{align*}
This procedure is known as $L_{\infty}$ regression or the Chebyshev estimator and has been studied in the statistics community~\citep{knight2017asymptotic,yi2024non}. 
These analyses primarily consider the well-specified setting with noise that is uniformly distributed, i.e., $y_i = f^\star(x_i) + \epsilon_i$ where $\epsilon_i \sim \mathrm{Unif}([-a,a])$ for some $a \geq 0$. 
We believe such analyses can extend to the $L_{\infty}$-misspecified setting to show that the procedure avoids misspecification amplification. 
However, strong assumptions on the noise are crucial, as $L_{\infty}$ regression can be inconsistent under more general conditions. 

We illustrate with a simple example. Let $\cX = \{x\}$ be a singleton, $y = \textrm{Ber}(\nicefrac{1}{4})$ and $\cF = \{ \fstar: x\mapsto 1/4, f: x \mapsto 1/2\}$ be a class with two functions. 
For all $n$ sufficiently large, the dataset will contain the sample $(x,1)$ at which point $f^\star$ will have $L_{\infty}$ error $3/4$, while $f$ will have error $1/2$.
Thus the method will be inconsistent.

\subsection{Analysis for DBR}
\label{app:dbr}
We begin with the proofs of~\pref{lem:nonnegative} and~\pref{lem:concentration}, thus completing steps one and two of the proof. Then we turn to proving the corollaries. 

\nonnegative*
\begin{proof}[Proof of~\pref{lem:nonnegative}]
Following the calculation used to derive~\pref{eq:sq_loss_mean} we have that, conditional on any $x$:
\begin{align*}
\Etrain\brk{ (f(x) - y)^2 - (\fbar(x) - y)^2 \mid x}  = (f(x) - \fstar(x))^2 -  (\fbar(x) - \fstar(x))^2
\end{align*}
Under the event $x \in W^\tau_{f,\fbar}$ with $\tau \geq 2\misspec$ we claim that this must be non-negative. In particular
\begin{align*}
|f(x) - \fstar(x)| \geq |f(x) - \fbar(x)| - | \fbar(x) - \fstar(x) | \geq \tau - \misspec \geq \misspec \geq 0
\end{align*}
Therefore,
\begin{align*}
(f(x) - \fstar(x))^2 - (\fbar(x) - \fstar(x))^2 \geq (\tau-\misspec)^2 - \misspec^2 \geq \tau^2 - 2\tau\misspec.
\end{align*}
The right hand side is non-negative whenever $\tau \geq 2\veps$.
\end{proof}

\concentration*
\begin{proof}[Proof of~\pref{lem:concentration}]
The concentration inequality is similar to the one used in the proof of~\pref{prop:erm_ub}.
We apply Bernstein's inequality and a union bound to the empirical disagreement-based loss $\cLhat(f;\fbar)$ for each $f \in \cF$. 
To do so, we must calculate the mean, variance, and range of $\cLhat(f;\fbar)$. 
Note that by the same calculation as in the proof of~\pref{prop:erm_ub}, we have that $\bbE[ \cLhat(f;\fbar) ] = \cL(f;\fbar)$ and that the range of each random variable in the empirical average is $1$. 
The variance calculation however is slightly different:
\begin{align*}
\Var\brk{W_{f,\fbar}^\tau(x) \cdot\crl{(f(x) -y)^2 - (\fbar(x) - y)^2} } & \leq \bbE\brk{ W_{f,\fbar}^\tau(x) \cdot\crl{(f(x) -y)^2 - (\fbar(x) - y)^2}^2 } \\
& \leq \bbE\brk{W_{f,\fbar}^\tau(x) (f(x) - \fbar(x))^2 (f(x) + \fbar(x) - 2y)^2 } \\
& \leq 4 \bbE\brk{W_{f,\fbar}^\tau(x) (f(x) - \fbar(x))^2}.
\end{align*}
Next, we consider a fixed $x$ and define $a := (f(x) - \fstar(x))$ and $b :=  (f^\star(x) - \fbar(x))$, so that we can write $(f(x) - \fbar(x))^2 = (f(x) - \fstar(x) + \fstar(x) - \fbar(x))^2 = (a+b)^2$. 
Now, when $\tau \geq 3\misspec$ we have:
\begin{align*}
W_{f,\fbar}^\tau(x) = 1 \Rightarrow |a| = |f(x) - \fstar(x)| \geq |f(x) - \fbar(x)| - \misspec \geq 2 \misspec.
\end{align*}
Along with the fact that $|b| = |\fbar(x) - \fstar(x)| \leq \misspec$, this implies that $|b| \leq |a|/2$ or equivalently that $b^2 \leq a^2/4$. 
Using this, we can deduce that
\begin{align*}
(a+b)^2 \leq  \frac{9a^2}{4} \leq \frac{9a^2}{4} - 3b^2 + \frac{3a^2}{2} = 3(a^2 - b^2).
\end{align*}
Re-introducing the definitions for $a$ and $b$ we have that
\begin{align*}
\Var\brk{W_{f,\fbar}^\tau(x) \cdot\crl{(f(x) -y)^2 - (\fbar(x) - y)^2} }\leq 12 \cL(f;\fbar)
\end{align*}
Now, applying Bernstein's inequality and a union bound over all $f \in \cF$ yields that with probability $1-\delta$:
\begin{align*}
\forall  f\in \cF: \cL(f;\fbar) - \cLhat(f;\fbar) & \leq \sqrt{ \frac{24\cL(f;\fbar) \log(|\cF|/\delta)}{n}} + \frac{4 \log(|\cF|/\delta)}{3n} \leq \frac{1}{2} \cL(f;\fbar) + \frac{40\log(|\cF|/\delta)}{3n}.
\end{align*}
Re-arranging proves the first statement, and the second statement follows from the symmetries $\cLhat(f;g) = -\cLhat(g;f)$ and $\cL(f;g) = -\cL(g;f)$. 
\end{proof}

\dbrcs*
\begin{proof}[Proof of~\pref{cor:dbr_cs}]
Beginning the with risk under $\Dtest$ and assuming that $\tau = 3\misspec$ we can write
\begin{align*}
\Rtest(\fdro) &= \Etest\brk{ (\fdro(x) - \fstar(x))^2 }\\
& = \Etest\brk{\ind\crl{ |\fdro(x) - \fstar(x)| < 4\misspec} \cdot (\fdro(x) - \fstar(x))^2 } \\
& ~~~~~+ \Etest \brk{\ind\crl{ |\fdro(x) - \fstar(x)| \geq 4\misspec} \cdot (\fdro(x) - \fstar(x))^2 } \\
& \leq 16\misspec^2 + \Etest \brk{\ind\crl{ |\fdro(x) - \fstar(x)| \geq 4\misspec} \cdot (\fdro(x) - \fstar(x))^2 }\\
& \leq 17\misspec^2 + \Etest \brk{ \ind\crl{ |\fdro(x) - \fstar(x)| \geq 4\misspec} \cdot \crl{ (\fdro(x) - \fstar(x))^2 - \misspec^2} }.
\end{align*}
Note that, due to the indicator, the quantity inside the expectation is non-negative. Therefore, via exactly the same importance weighting argument as we used in the proof of~\pref{prop:erm_ub}, the latter is at most $\Cinf$ times the quantity bounded in~\pref{eq:dis_risk}.
\end{proof}

\dbrrealizable*
\begin{proof}[Proof of~\pref{cor:dbr_realizable}]
Let $\Delta$ denote the right hand side of~\pref{eq:dis_risk}.
Note that in the well-specified case where $\misspec = 0$, ~\pref{thm:main} ensures that
\begin{align*}
\Etrain\brk{\ind\crl{ |\fdro(x) - \fstar(x)| \geq \tau}\cdot (\fdro(x) - \fstar(x))^2 } \leq \Delta. 
\end{align*}
Then, if we take $\tau \leq \sqrt{\Delta}$, we have
\begin{align*}
\Rtrain(\fdro) &= \Etrain\brk{\ind\crl{ |\fdro(x) - \fstar(x)| < \tau} \cdot (\fdro(x) - \fstar(x))^2 } \\
& ~~~~~+ \Etrain \brk{\ind\crl{ |\fdro(x) - \fstar(x)| \geq \tau} \cdot (\fdro(x) - \fstar(x))^2 } \\
& \leq \tau^2 + \Delta \leq 2\Delta. 
\end{align*}
This proves the corollary. 
\end{proof}

\subsection{Extensions}
\label{app:extensions}
In this section, we provide two results mentioned in~\pref{sec:regression}.
First we improve the approximation factor in~\pref{cor:dbr_cs} from $17$ to $10$ albeit at the cost of a worse statistical term.
Second we show how to choose $\tau$ in a data-driven fashion to adapt to unknown misspecification level $\misspec$.

\begin{proposition}[Improved approximation factor]\label{prop:improved-approx}
Under~\pref{assum:covariate_shift}--\pref{assum:boundedness}, with $\tau = 2\misspec$ and for $\delta \in (0,1)$, we have that, with probability at least $1-\delta$:
\begin{align}
\Rtest(\fdro) \leq  10\misspec^2 + \Cinf \cdot O\prn*{ \sqrt{\frac{\log (|\cF|/\delta)}{n}}}.
\end{align}\end{proposition}
\begin{proof}[Proof sketch]
The proof is essentially identical to that of~\pref{thm:main}, except that we replace the concentration statement of~\pref{lem:concentration} with a simpler one that relies on Hoeffding's inequality. 
The new concentration statement is that for any $\tau\geq 0$ and $\delta \in (0,1)$ with probability $1-\delta$ we have
\begin{align*}
\forall f \in \cF: ~~ \cL(f;\fbar) \leq \cLhat(f;\fbar) + \epsslow,
\end{align*}
where $\epsslow := c \sqrt{ \frac{\log (|\cF|/\delta)}{n}}$ for some universal constant $c>0$.
This follows by a standard application of Hoeffding's inequality and a union bound, but importantly does not impose the restriction that $\tau \geq 3\misspec$. 

Now the analysis to prove~\pref{thm:main} yields that for any $\tau \geq 2\misspec$:
\begin{align*}
\Etrain\brk*{\ind\{ |\fdro(x) - \fstar(x)|\geq \tau+\misspec \}\cdot \crl*{(\fdro(x) - \fstar(x))^2 - (\fbar(x) - \fstar(x))^2 } } \leq c \epsslow.
\end{align*}
Taking $\tau = 2\misspec$ and following the derivation used to prove~\pref{cor:dbr_cs}, we get
\begin{align*}
\Rtest(\fdro) \leq 10 \misspec^2 + c\epsslow
\end{align*}
(Note that this requires the non-negativity property provided by~\pref{lem:nonnegative}, which we still have.)
\end{proof}

The next result considers adapting to an unknown misspecification level.
\begin{proposition}[Adapting to $\misspec$]\label{prop:adapting-misspeci}
Let $\delta \in (0,1)$ and define $S := \{ 2^i: \tau_{\min} \leq 2^i \leq \tau_{\max}\}$ where $\tau_{\min} := \sqrt{ \frac{160 \log (|\cF| |S|/\delta)}{3n}}$ and $\tau_{\max} := 1$. 
Let $\tau^\star := \min \{ \tau \in S: \tau \geq 3\misspec\}$. 
Then there is an algorithm that, without knowledge of $\misspec$ and with probability at least $1-\delta$, computes $\hat{f}$ satisfying
\begin{align*}
\Etrain\brk*{\ind\{ |\fhat(x) - \fstar(x)|\geq \tau^\star+\misspec \}\cdot \crl*{(\fhat(x) - \fstar(x))^2 - \misspec^2 } } \leq \frac{160 \log (2|\cF||S|/\delta)}{3n}.
\end{align*}
\end{proposition}
Note that when $\misspec \ll \tau_{\min}$, we are essentially in the realizable regime. Thus, via the proof of~\pref{cor:dbr_realizable} the above guarantee with $\tau^\star := \tau_{\min}$ suffices.
On the other hand if $\misspec \geq 1/3$ then $\tau^\star$ is undefined, but due to~\pref{assum:boundedness} the guarantee in~\pref{thm:main} is vacuous. 
Thus, the above theorem recovers essentially the same result as~\pref{thm:main}, but without knowledge of $\misspec$. 
\begin{proof}[Proof sketch]
The algorithm is as follows. 
We run a slight variation of disagreement based regression for each $\tau \in S$: Instead of computing the minimizer of the objective in~\pref{eq:main_alg} we form the version space of near-minimizers.
Specifically, define
\begin{align*}
\forall \tau \in S:~~  F_\tau := \crl*{ f \in \cF: \max_{g \in \cF} \frac{1}{n}\sum_{i=1}^n W_{f,g}^\tau(x_i) \crl*{(f(x_i) -y_i)^2 - (g(x_i)-y_i)^2} \leq \epsstat/2 },
\end{align*}
where we define $\epsstat = \frac{80 \log(|\cF||S|/\delta)}{3n}$. 
Note this is slightly inflated from the definition in the statement of~\pref{lem:concentration}, which accounts for a union bound over all $|S|$ runs of the algorithm. 
Next, we define
\begin{align*}
\hat{\tau} := \argmin\crl*{ \tau \in S: \bigcap_{\tau' \in S: \tau' \geq \tau} F_{\tau'} \neq \emptyset},
\end{align*}
and return any function in this intersection, i.e., let $\hat{f}$ be any function in $\bigcap_{\tau' \in S: \tau' \geq \hat{\tau}}  F_{\tau'}$. 

For the analysis, via the analysis of~\pref{thm:main} and a union bound over the $|S|$ choices for $\tau$, we have
\begin{align*}
\forall \tau \geq \tau^\star: ~~~ \fbar \in F_{\tau} \quad \textrm{ and } \quad f \in F_\tau \Rightarrow \cL^\tau(f;\fbar) \leq \epsstat,
\end{align*}
where $\cL^\tau(f;g)$ is the population objective with parameter $\tau$. 
The first statement directly implies that $\hat{\tau} \leq \tau^\star$. 
This in turn implies that $\hat{f} \in F_{\tau^\star}$ and so $\hat{f}$ achieves the same statistical guarantee as if we ran \dbr with parameter $\tau^\star$ (up to the additional union bound). 
\end{proof}

%% file: appendix_rl.tex
\subsection{Offline RL}

\offlinerl*

\begin{proof}[\pfref{thm:offline_rl}]
For each ``target'' function $\ftrg \in \cF$ such that $\ftrg \ne \fbar$, let us define $\apx[\ftrg] \in \cF$ to be \emph{any} approximation to the Bellman backup $\cT \ftrg$ s.t. $\| \apx[\ftrg] - \cT \ftrg \|_{\infty} \leq \misspec$. Define $\apx[\fbar] = \fbar$, which also satisfies $\|\apx[\fbar] - \cT\fbar\|_{\infty} \leq \misspec$ by assumption.
Let us define the empirical and population losses for the disagreement-based regression problem with regression targets derived from $\ftrg$.
\begin{align*}
\textrm{(Empirical)}~~\cLhat_{\ftrg}(f;g) &:= \frac{1}{n}\sum_{i=1}^n W_{f,g}^\tau(s_i,a_i)\crl*{ (f(s_i,a_i) - y_{\ftrg,i})^2 - (g(s_i,a_i) - y_{\ftrg,i})^2},\\
\textrm{(Population)}~~\cL_{\ftrg}(f;g) &:= \bbE_{\mu}\brk*{W_{f,g}^\tau(s,a)\crl*{ (f(s,a) - y_{\ftrg})^2 - (g(s,a) - y_{\ftrg})^2} }.
\end{align*}
Here recall that $y_{\ftrg} := r + \max_{a'} \ftrg(s',a')$ is derived from the sample $(s,a,r,s')$. 
Also note that we use $\bbE_{\mu}\brk*{\cdot}$ to denote expectation with respect to the data collection policy. 

First, we apply~\pref{lem:nonnegative} and~\pref{lem:concentration} to each of the $|\cF|$ regression problems. 
By approximate completeness and the definition of $\apx[\ftrg]$ this yields
\begin{align}
\forall \ftrg,f \in \cF: ~~ 0 \leq \cL_{\ftrg}(f; \apx[\ftrg]) \leq 2\cLhat_{\ftrg}(f;\apx[\ftrg]) + \epsstat, \label{eq:offline_rl_conc}
\end{align}
where $\epsstat := \frac{160\log(|\cF|/\delta)}{3n}$. The above uniform bound holds with probability $1-\delta$. 
Note that this $\epsstat$ is twice as large as the one in the proof of~\pref{thm:main}, which accounts for the additional union bound over all $|\cF|$ regression problems. 

The main statistical guarantee for $\fhat$ is derived as follows
\begin{align*}
\cL_{\fhat}(\fhat;\apx[\fhat]) & \leqlab{\mathrm{(i)}} 2\cLhat_{\fhat}(\fhat; \apx[\fhat]) + \epsstat \leqlab{\mathrm{(ii)}} 2 \max_{g \in \cF} \cLhat_{\fhat}(\fhat; g) + \epsstat \leqlab{\mathrm{(iii)}} 2 \max_{g \in \cF} \cLhat_{\fbar}(\fbar; g) + \epsstat \leqlab{\mathrm{(iv)}} 2\epsstat.
\end{align*}
Here $\mathrm{(i)}$ is the second inequality in~\pref{eq:offline_rl_conc}, $\mathrm{(ii)}$ follows since $\apx[\hat{f}] \in \cF$, $\mathrm{(iii)}$ uses the optimality property of $\hat{f}$, and $\mathrm{(iv)}$ uses~\pref{eq:offline_rl_conc} again, noting the symmetry of $\cL_{\fbar}(\cdot;\cdot)$ and using $\apx[\fbar]=\fbar$.

Since the Bayes regression function defined by targets $y_{\hat{f}}$ is $\cT \hat{f}$, this yields
\begin{align}
\bbE_{\mu}\brk*{\ind\crl*{ \abs{ \hat{f}(s,a) - \apx[\fhat](s,a)} \geq 3\misspec  } \cdot \crl*{ (\hat{f}(s,a) - [\cT \hat{f}](s,a))^2 - (\apx[\hat{f}](s,a) - [\cT\hat{f}](s,a))^2 } } \leq 2\epsstat. \label{eq:offline_rl_excess_risk}
\end{align}

We translate this to the squared Bellman error on any other distribution $\nu \in \Delta(\cX\times\cA)$ via a slightly stronger argument than the one used to prove~\pref{cor:dbr_cs}.
\begin{align*}
& \bbE_{\nu}\brk*{ \abs*{ \fhat(s,a) - [\cT\fhat](s,a) }} \leq \misspec + \bbE_{\nu}\brk*{ \abs*{ \fhat(s,a) - \apx[\fhat](s,a) }}\\
& ~~~~ \leq 4\misspec + \bbE_{\nu}\brk*{ \ind\crl*{\abs{ \hat{f}(s,a) - \apx[\fhat](s,a)} \geq 3\misspec }\cdot\abs*{ \fhat(s,a) - \apx[\fhat](s,a) }}\\
& ~~~~ \leq 4\misspec + \sqrt{ \bbE_{\mu}\brk*{ \prn*{ \frac{\nu(s,a)}{\mu(s,a)}}^2 }} \cdot \sqrt{\bbE_{\mu}\brk*{ \ind\crl*{\abs{ \hat{f}(s,a) - \apx[\fhat](s,a)} \geq 3\misspec }\cdot\prn*{ \fhat(s,a) - \apx[\fhat](s,a) }^2}}\\
& ~~~~ = 4\misspec + \| \nu/\mu \|_{L_2(\mu)} \cdot \sqrt{\bbE_{\mu}\brk*{ \ind\crl*{\abs{ \hat{f}(s,a) - \apx[\fhat](s,a)} \geq 3\misspec }\cdot\prn*{ \fhat(s,a) - \apx[\fhat](s,a) }^2}}\\
& ~~~~ \leq 4\misspec + \nrm*{ \nu/\mu }_{L_2(\mu)} \cdot  \sqrt{ 6 \epsstat}.
\end{align*}
The last inequality is based on the ``self-bounding'' argument we used to control the variance in the proof of~\pref{lem:concentration}, which showed that under the event $\abs{\hat{f}(s,a) - \apx[\fhat](s,a)} \geq 3\misspec$:
\begin{align*}
\prn*{ \fhat(s,a) - \apx[\fhat](s,a) }^2 \leq 3\cdot \crl*{ \prn*{ \fhat(s,a) - [\cT \fhat](s,a) }^2 - \prn*{ \apx[\fhat](s,a) -[\cT\fhat](s,a)}^2 }.
\end{align*}
Note that $\nrm*{\nu/\mu}_{L_2(\mu)}^2 \leq  \nrm*{ \nu/\mu}_{\infty}$ since $\bbE_{\mu}[ \nu(s,a)/\mu(s,a)] = \bbE_{\nu}[1] = 1$.

Finally, we appeal to the telescoping performance difference lemma~\citep[c.f.,][Theorem 2]{xie2020q}, which states that for an action-value function $f$, 
\begin{align*}
J(\pi^\star) - J(\pi_f) \leq \frac{\bbE_{d^{\pi^\star}}[ [\cT f](s,a) - f(s,a)]}{1-\gamma} + \frac{\bbE_{d^{\pi_f}}[ f(s,a) - [\cT f](s,a)] }{1-\gamma},
\end{align*}
where $d^\pi := (1-\gamma) \sum_{h=0}^{\infty} \gamma^h d_h^\pi$.  Both terms are controlled by the distribution shift argument above and the concentrability coefficient, yielding the theorem. 
\end{proof}

\subsection{Online RL}

\onlinerl*

\begin{proof}[\pfref{thm:online_rl}]
The proof makes essentially two modifications to the proof of Theorem 1 of~\citet{xie2022role}. 
The first step is a concentration argument, which is essentially a martingale version of~\pref{thm:main}.
The second is the distribution shift argument, which is very similar to the one we used to prove~\pref{thm:offline_rl}.
To keep the presentation concise, we focus on these arguments, and explain how they fit into the analysis of~\citet{xie2022role}, but we do not provide a self-contained proof. 

\paragraph{Notation}
We adopt the following notation. 
Recall that $\cF\index{t-1}$ is the version space used in episode $t$ and that $f\index{t} \in \cF\index{t-1}$ induces the policy $\pit$ deployed in the episode. 
As before, let $\apx[f_{h+1}] \in \cF_h$ denote the $L_{\infty}$-approximation to $\cT_h f_{h+1}$. 
For each episode $t$ let
\begin{align*}
\delta_h\index{t}(\cdot) &:= f_h\index{t}(\cdot) - [\cT_h f_{h+1}\index{t}](\cdot) \qquad \textrm{and}\\
\err\index{t}(\cdot) &:= \ind\crl*{ \abs*{f_h\index{t}(\cdot) - \apx[f_{h+1}\index{t}](\cdot)} \geq 3\misspec}\cdot\crl*{\prn*{ f_h\index{t}(\cdot) - [\cT_h f_{h+1}\index{t}](\cdot)}^2 -  \prn*{ \apx[f_{h+1}\index{t}](\cdot) - [\cT_h f_{h+1}\index{t}](\cdot)}^2}.
\end{align*}
Let $d_h\index{t} = d_h^{\pi\index{t}}$ and define $\widetilde{d}_h\index{t}(x,a) = \sum_{i=1}^{t-1}d_h\index{i}(x,a)$ and $\mu_h^\star$ to be the distribution that achieves the value $\Ccov$ for layer $h$.

\paragraph{Concentration}
By a martingale version of~\pref{thm:main}, we can show that with probability at least $1-\delta$, for all $t \in [T]$:
\begin{align}
(i)~~ \fbar \in \cF\index{t}, \quad \mathrm{ and } \quad (ii)~~\forall h \in [H]:~~ \sum_{s,a} \widetilde{d}_h\index{t}(s,a) \err_{h}\index{t}(s,a) \leq O(\beta), \label{eq:online_rl_conc}
\end{align}
where $\beta = c\log(TH|\cF|/\delta)$. 
We do not provide a complete proof of this statement, noting that it is essentially the same guarantee as in~\pref{eq:offline_rl_excess_risk}, except that (a) it is a non-stationary version with a union bound over each time step $h$ and episode $t$ and (b) it uses martingale concentration (i.e., Freedman's inequality instead of Bernstein's inequality). 
It is also worth comparing with the concentration guarantee of~\citep{xie2022role} under exact realizability/completeness, which is that $Q^\star \in \cF\index{t}$ and that $\sum_{s,a} \widetilde{d}_h\index{t}(s,a)(\delta_h\index{t}(s,a))^2 \leq O(\beta)$. 

\paragraph{Distribution shift}
To bound the regret, note that
\begin{align*}
\reg \leq \sum_{t=1}^T \sum_{h=1}^H \bbE_{(s,a) \sim d_h\index{t}}\brk*{\delta_h\index{t}(s,a) }.
\end{align*}

For distribution shift, we must translate the above on-policy Bellman errors to the ``\dbr'' errors on the historical data $\widetilde{d}_h\index{t}$, which is controlled by~\pref{eq:online_rl_conc}. 
Following~\citep{xie2022role} we consider \emph{burn-in} and \emph{stable} phases. Let

\begin{align*}
\gamma_h(s,a) := \min\crl*{t : \widetilde{d}_h\index{t}(s,a) \geq \Ccov\cdot \mu_h^\star(s,a) }, 
\end{align*}
and decompose
\begin{align*}
\sum_{t=1}^T \bbE_{(s,a) \sim d_h\index{t}}\brk*{\delta_h\index{t}(s,a) } = \sum_{t=1}^T \bbE_{(s,a) \sim d_h\index{t}}\brk*{\delta_h\index{t}(s,a) \ind\crl*{t < \gamma_h(s,a)}} + \bbE_{(s,a) \sim d_h\index{t}}\brk*{\delta_h\index{t}(s,a) \ind\crl*{t \geq \gamma_h(s,a)}}.
\end{align*}
The first term is the regret incurred during the burn-in phase, which is bounded by $2\Ccov$ following exactly the argument of~\citet{xie2022role}. 
This contributes a total regret of $2H\Ccov$. 

The second term is the regret incurred during the stable-phase, for which we must perform a distribution shift argument. 
To condense the notation, define
\begin{align*}
\bar{\delta}_h\index{t}(\cdot) := \apx[f_{h+1}\index{t}](\cdot) - [\cT_h f_{h+1}\index{t}](\cdot), \quad \mathrm{ and } \quad \tilde{\delta}_h\index{t}(\cdot) := f_h\index{t}(\cdot) - \apx[f_{h+1}\index{t}](\cdot).
\end{align*}
Note that, by assumption, $\abs*{\bar{\delta}_h\index{t}(s,a)} \leq \misspec$. 
Then,
\begin{align*}
& \sum_{t=1}^T \bbE_{d_h\index{t}}\brk*{\delta_h\index{t}(s,a) \ind\crl*{t > \gamma_h(s,a)}}\\
& ~~~~ = \sum_{t=1}^T \bbE_{d_h\index{t}}\brk*{\prn*{\tilde{\delta}_h\index{t}(s,a) + \bar{\delta}_{h}\index{t}(s,a) }\ind\crl*{t > \gamma_h(s,a)}}\\
& ~~~~ \leq \sum_{t=1}^T \bbE_{d_h\index{t}}\brk*{\tilde{\delta}_h\index{t}(s,a) \ind\crl*{t > \gamma_h(s,a)}} + T\misspec\\
& ~~~~ \leq \sum_{t=1}^T\bbE_{d_h\index{t}}\brk*{\ind\crl*{|\tilde{\delta}_h\index{t}(s,a)| \geq 3\misspec}\tilde{\delta}_h\index{t}(s,a) \ind\crl*{t > \gamma_h(s,a)}} + 4T\misspec\\
& ~~~~ \leq \sqrt{\sum_{t=1}^T \sum_{s,a} \frac{\prn*{\ind\crl*{t > \gamma_h(s,a)}d_h\index{t}(s,a) }^2}{\widetilde{d}_h\index{t}(s,a)}}\cdot\sqrt{ \sum_{t=1}^T \sum_{x,a} \widetilde{d}_h\index{t}(x,a) \ind\crl*{|\tilde{\delta}_h\index{t}(s,a)| \geq 3\misspec}(\tilde{\delta}_h\index{t}(s,a))^2} + 4T\misspec\\
& ~~~~ \leq \sqrt{\sum_{t=1}^T \sum_{s,a} \frac{\prn*{\ind\crl*{t > \gamma_h(s,a)}d_h\index{t}(s,a) }^2}{\widetilde{d}_h\index{t}(s,a)}}\cdot\sqrt{ 3 \sum_{t=1}^T \sum_{x,a} \widetilde{d}_h\index{t}(x,a) \err_h\index{t}(s,a)} + 4T\misspec.
\end{align*}
The penultimate inequality is Cauchy-Schwarz and the final inequality follows from the self-bounding property that we used in the proof of~\pref{lem:concentration} and~\pref{thm:offline_rl}. 
In particular under the event that $\abs*{ \tilde{\delta}_{h}\index{t}(s,a) }\geq 3\misspec$, we can bound $(\tilde{\delta}_h\index{t}(s,a))^2 \leq 3\prn*{(\delta_h\index{t}(s,a))^2 - (\bar{\delta}_h\index{t}(s,a))^2}$. 
Thus we have converted from the on-policy Bellman error to the historical ``\dbr'' errors, i.e., we can further bound by
\begin{align*}
\leq \sqrt{\sum_{t=1}^T \sum_{s,a} \frac{\prn*{\ind\crl*{t > \gamma_h(s,a)}d_h\index{t}(s,a) }^2}{\widetilde{d}_h\index{t}(s,a)}}\cdot O\prn*{\sqrt{ \beta T} } + 4T\misspec.
\end{align*}
Meanwhile the density ratio term is bounded by $O(\sqrt{ \Ccov\log(T) })$ via the analysis of~\citet{xie2022role}. 
Repeating this analysis for each time step $h$ proves the theorem. 
\end{proof}

%% file: arxiv.bbl
\begin{thebibliography}{66}
\providecommand{\natexlab}[1]{#1}
\providecommand{\url}[1]{\texttt{#1}}
\expandafter\ifx\csname urlstyle\endcsname\relax
  \providecommand{\doi}[1]{doi: #1}\else
  \providecommand{\doi}{doi: \begingroup \urlstyle{rm}\Url}\fi

\bibitem[Agarwal and Zhang(2022)]{agarwal2022minimax}
Alekh Agarwal and Tong Zhang.
\newblock Minimax regret optimization for robust machine learning under
  distribution shift.
\newblock In \emph{Conference on Learning Theory}, 2022.

\bibitem[Agarwal et~al.(2019)Agarwal, Jiang, Kakade, and
  Sun]{agarwal2019reinforcement}
Alekh Agarwal, Nan Jiang, Sham~M Kakade, and Wen Sun.
\newblock Reinforcement learning: Theory and algorithms.
\newblock \url{https://rltheorybook.github.io/}, 2019.
\newblock Version: January 31, 2022.

\bibitem[Amortila et~al.(2023)Amortila, Jiang, and
  Szepesv{\'a}ri]{amortila2023optimal}
Philip Amortila, Nan Jiang, and Csaba Szepesv{\'a}ri.
\newblock The optimal approximation factors in misspecified off-policy value
  function estimation.
\newblock In \emph{International Conference on Machine Learning}, 2023.

\bibitem[Anil et~al.(2022)Anil, Wu, Andreassen, Lewkowycz, Misra, Ramasesh,
  Slone, Gur-Ari, Dyer, and Neyshabur]{anil2022exploring}
Cem Anil, Yuhuai Wu, Anders Andreassen, Aitor Lewkowycz, Vedant Misra, Vinay
  Ramasesh, Ambrose Slone, Guy Gur-Ari, Ethan Dyer, and Behnam Neyshabur.
\newblock Exploring length generalization in large language models.
\newblock \emph{Advances in Neural Information Processing Systems}, 2022.

\bibitem[Antos et~al.(2008)Antos, Szepesv{\'a}ri, and Munos]{antos2008learning}
Andr{\'a}s Antos, Csaba Szepesv{\'a}ri, and R{\'e}mi Munos.
\newblock Learning near-optimal policies with bellman-residual minimization
  based fitted policy iteration and a single sample path.
\newblock \emph{Machine Learning}, 2008.

\bibitem[Audibert(2007)]{audibert2007progressive}
Jean-Yves Audibert.
\newblock Progressive mixture rules are deviation suboptimal.
\newblock \emph{Advances in Neural Information Processing Systems}, 2007.

\bibitem[Ben-David et~al.(2006)Ben-David, Blitzer, Crammer, and
  Pereira]{ben2006analysis}
Shai Ben-David, John Blitzer, Koby Crammer, and Fernando Pereira.
\newblock Analysis of representations for domain adaptation.
\newblock \emph{Advances in Neural Information Processing Systems}, 19, 2006.

\bibitem[Besbes et~al.(2015)Besbes, Gur, and Zeevi]{besbes2015non}
Omar Besbes, Yonatan Gur, and Assaf Zeevi.
\newblock Non-stationary stochastic optimization.
\newblock \emph{Operations Research}, 2015.

\bibitem[Bradtke and Barto(1996)]{bradtke1996linear}
Steven~J Bradtke and Andrew~G Barto.
\newblock Linear least-squares algorithms for temporal difference learning.
\newblock \emph{Machine learning}, 1996.

\bibitem[Chen and Jiang(2019)]{chen2019information}
Jinglin Chen and Nan Jiang.
\newblock Information-theoretic considerations in batch reinforcement learning.
\newblock In \emph{International Conference on Machine Learning}, 2019.

\bibitem[Cortes and Mohri(2014)]{cortes2014domain}
Corinna Cortes and Mehryar Mohri.
\newblock Domain adaptation and sample bias correction theory and algorithm for
  regression.
\newblock \emph{Theoretical Computer Science}, 2014.

\bibitem[Cortes et~al.(2010)Cortes, Mansour, and Mohri]{cortes2010learning}
Corinna Cortes, Yishay Mansour, and Mehryar Mohri.
\newblock Learning bounds for importance weighting.
\newblock \emph{Advances in Neural Information Processing Systems}, 2010.

\bibitem[Dong and Ma(2023{\natexlab{a}})]{dong2022first}
Kefan Dong and Tengyu Ma.
\newblock First steps toward understanding the extrapolation of nonlinear
  models to unseen domains.
\newblock In \emph{International Conference on Learning Representations},
  2023{\natexlab{a}}.

\bibitem[Dong and Ma(2023{\natexlab{b}})]{dong2023toward}
Kefan Dong and Tengyu Ma.
\newblock Toward ${L}_{\infty}$-recovery of nonlinear functions: A polynomial
  sample complexity bound for gaussian random fields.
\newblock In \emph{Conference on Learning Theory}, 2023{\natexlab{b}}.

\bibitem[Du et~al.(2021)Du, Kakade, Lee, Lovett, Mahajan, Sun, and
  Wang]{du2021bilinear}
Simon Du, Sham Kakade, Jason Lee, Shachar Lovett, Gaurav Mahajan, Wen Sun, and
  Ruosong Wang.
\newblock Bilinear classes: {A} structural framework for provable
  generalization in {RL}.
\newblock In \emph{International Conference on Machine Learning}, 2021.

\bibitem[Du et~al.(2020)Du, Kakade, Wang, and Yang]{du2019good}
Simon~S Du, Sham~M Kakade, Ruosong Wang, and Lin~F Yang.
\newblock Is a good representation sufficient for sample efficient
  reinforcement learning?
\newblock In \emph{International Conference on Learning Representations}, 2020.

\bibitem[Duchi and Namkoong(2021)]{duchi2021learning}
John~C Duchi and Hongseok Namkoong.
\newblock Learning models with uniform performance via distributionally robust
  optimization.
\newblock \emph{The Annals of Statistics}, 2021.

\bibitem[Foster et~al.(2018)Foster, Agarwal, Dud{\'\i}k, Luo, and
  Schapire]{foster2018practical}
Dylan Foster, Alekh Agarwal, Miroslav Dud{\'\i}k, Haipeng Luo, and Robert
  Schapire.
\newblock Practical contextual bandits with regression oracles.
\newblock In \emph{International Conference on Machine Learning}, 2018.

\bibitem[Foster and Rakhlin(2023)]{foster2023foundations}
Dylan~J Foster and Alexander Rakhlin.
\newblock Foundations of reinforcement learning and interactive decision
  making.
\newblock \emph{arXiv:2312.16730}, 2023.

\bibitem[Foster et~al.(2021)Foster, Rakhlin, Simchi-Levi, and
  Xu]{foster2020instance}
Dylan~J Foster, Alexander Rakhlin, David Simchi-Levi, and Yunzong Xu.
\newblock Instance-dependent complexity of contextual bandits and reinforcement
  learning: A disagreement-based perspective.
\newblock In \emph{Conference on Learning Theory}, 2021.

\bibitem[Foster et~al.(2022)Foster, Krishnamurthy, Simchi-Levi, and
  Xu]{foster2021offline}
Dylan~J Foster, Akshay Krishnamurthy, David Simchi-Levi, and Yunzong Xu.
\newblock Offline reinforcement learning: {F}undamental barriers for value
  function approximation.
\newblock In \emph{Conference on Learning Theory}, 2022.

\bibitem[Gama et~al.(2014)Gama, {\v{Z}}liobait{\.e}, Bifet, Pechenizkiy, and
  Bouchachia]{gama2014survey}
Jo{\~a}o Gama, Indr{\.e} {\v{Z}}liobait{\.e}, Albert Bifet, Mykola Pechenizkiy,
  and Abdelhamid Bouchachia.
\newblock A survey on concept drift adaptation.
\newblock \emph{ACM Computing Surveys}, 2014.

\bibitem[Ge et~al.(2023)Ge, Tang, Fan, Ma, and Jin]{ge2023maximum}
Jiawei Ge, Shange Tang, Jianqing Fan, Cong Ma, and Chi Jin.
\newblock Maximum likelihood estimation is all you need for well-specified
  covariate shift.
\newblock \emph{arXiv:2311.15961}, 2023.

\bibitem[Gretton et~al.(2009)Gretton, Smola, Huang, Schmittfull, Borgwardt, and
  Sch{\"o}lkopf]{gretton2009covariate}
Arthur Gretton, Alex Smola, Jiayuan Huang, Marcel Schmittfull, Karsten
  Borgwardt, and Bernhard Sch{\"o}lkopf.
\newblock Covariate shift by kernel mean matching.
\newblock \emph{Dataset Shift in Machine Learning}, 2009.

\bibitem[Hanneke(2014)]{hanneke2014theory}
Steve Hanneke.
\newblock Theory of disagreement-based active learning.
\newblock \emph{Foundations and Trends in Machine Learning}, 2014.

\bibitem[Hashimoto et~al.(2018)Hashimoto, Srivastava, Namkoong, and
  Liang]{hashimoto2018fairness}
Tatsunori Hashimoto, Megha Srivastava, Hongseok Namkoong, and Percy Liang.
\newblock Fairness without demographics in repeated loss minimization.
\newblock In \emph{International Conference on Machine Learning}, 2018.

\bibitem[Huang et~al.(2006)Huang, Gretton, Borgwardt, Sch{\"o}lkopf, and
  Smola]{huang2006correcting}
Jiayuan Huang, Arthur Gretton, Karsten Borgwardt, Bernhard Sch{\"o}lkopf, and
  Alex Smola.
\newblock Correcting sample selection bias by unlabeled data.
\newblock \emph{Advances in Neural Information Processing Systems}, 2006.

\bibitem[Jiang et~al.(2017)Jiang, Krishnamurthy, Agarwal, Langford, and
  Schapire]{jiang2017contextual}
Nan Jiang, Akshay Krishnamurthy, Alekh Agarwal, John Langford, and Robert~E
  Schapire.
\newblock Contextual decision processes with low {B}ellman rank are
  {PAC}-learnable.
\newblock In \emph{International Conference on Machine Learning}, 2017.

\bibitem[Jin et~al.(2021)Jin, Liu, and Miryoosefi]{jin2021bellman}
Chi Jin, Qinghua Liu, and Sobhan Miryoosefi.
\newblock Bellman eluder dimension: {N}ew rich classes of {RL} problems, and
  sample-efficient algorithms.
\newblock \emph{Advances in Neural Information Processing Systems}, 2021.

\bibitem[Knight(2017)]{knight2017asymptotic}
Keith Knight.
\newblock On the asymptotic distribution of the ${L}_{\infty}$ estimator in
  linear regression.
\newblock Technical report, University of Toronto, 2017.

\bibitem[Koh et~al.(2021)Koh, Sagawa, Marklund, Xie, Zhang, Balsubramani, Hu,
  Yasunaga, Phillips, Gao, Lee, David, Stavness, Guo, Earnshaw, Haque, Beery,
  Leskovec, Kundaje, Pierson, Levine, Finn, and Liang]{koh2021wilds}
Pang~Wei Koh, Shiori Sagawa, Henrik Marklund, Sang~Michael Xie, Marvin Zhang,
  Akshay Balsubramani, Weihua Hu, Michihiro Yasunaga, Richard~Lanas Phillips,
  Irena Gao, Tony Lee, Etienne David, Ian Stavness, Wei Guo, Berton Earnshaw,
  Imran Haque, Sara~M Beery, Jure Leskovec, Anshul Kundaje, Emma Pierson,
  Sergey Levine, Chelsea Finn, and Percy Liang.
\newblock Wilds: A benchmark of in-the-wild distribution shifts.
\newblock In \emph{International Conference on Machine Learning}, 2021.

\bibitem[Kpotufe and Martinet(2018)]{kpotufe2018marginal}
Samory Kpotufe and Guillaume Martinet.
\newblock Marginal singularity, and the benefits of labels in covariate-shift.
\newblock In \emph{Conference On Learning Theory}, 2018.

\bibitem[Krishnamurthy et~al.(2019)Krishnamurthy, Agarwal, Huang,
  Daum{\'e}~III, and Langford]{krishnamurthy2019active}
Akshay Krishnamurthy, Alekh Agarwal, Tzu-Kuo Huang, Hal Daum{\'e}~III, and John
  Langford.
\newblock Active learning for cost-sensitive classification.
\newblock \emph{Journal of Machine Learning Research}, 2019.

\bibitem[Lattimore et~al.(2020)Lattimore, Szepesvari, and
  Weisz]{lattimore2020learning}
Tor Lattimore, Csaba Szepesvari, and Gellert Weisz.
\newblock Learning with good feature representations in bandits and in {RL}
  with a generative model.
\newblock In \emph{International Conference on Machine Learning}, 2020.

\bibitem[Lecu{\'e} and Rigollet(2014)]{lecue2014optimal}
Guillaume Lecu{\'e} and Philippe Rigollet.
\newblock Optimal learning with {Q}-aggregation.
\newblock \emph{The Annals of Statistics}, 2014.

\bibitem[Lei et~al.(2021)Lei, Hu, and Lee]{lei2021near}
Qi~Lei, Wei Hu, and Jason Lee.
\newblock Near-optimal linear regression under distribution shift.
\newblock In \emph{International Conference on Machine Learning}, 2021.

\bibitem[Levine et~al.(2020)Levine, Kumar, Tucker, and Fu]{levine2020offline}
Sergey Levine, Aviral Kumar, George Tucker, and Justin Fu.
\newblock Offline reinforcement learning: Tutorial, review, and perspectives on
  open problems.
\newblock \emph{arXiv:2005.01643}, 2020.

\bibitem[Liang et~al.(2015)Liang, Rakhlin, and Sridharan]{liang2015learning}
Tengyuan Liang, Alexander Rakhlin, and Karthik Sridharan.
\newblock Learning with square loss: Localization through offset rademacher
  complexity.
\newblock In \emph{Conference on Learning Theory}, 2015.

\bibitem[Liu et~al.(2023)Liu, Ash, Goel, Krishnamurthy, and
  Zhang]{liu2023exposing}
Bingbin Liu, Jordan~T Ash, Surbhi Goel, Akshay Krishnamurthy, and Cyril Zhang.
\newblock Exposing attention glitches with flip-flop language modeling.
\newblock \emph{Advances in Neural Information Processing Systems}, 2023.

\bibitem[Ma et~al.(2023)Ma, Pathak, and Wainwright]{ma2023optimally}
Cong Ma, Reese Pathak, and Martin~J Wainwright.
\newblock Optimally tackling covariate shift in {RKHS}-based nonparametric
  regression.
\newblock \emph{The Annals of Statistics}, 2023.

\bibitem[Mansour et~al.(2009)Mansour, Mohri, and
  Rostamizadeh]{mansour2009domain}
Yishay Mansour, Mehryar Mohri, and Afshin Rostamizadeh.
\newblock Domain adaptation: {L}earning bounds and algorithms.
\newblock \emph{arXiv:0902.3430}, 2009.

\bibitem[Miller et~al.(2021)Miller, Taori, Raghunathan, Sagawa, Koh, Shankar,
  Liang, Carmon, and Schmidt]{miller2021accuracy}
John~P Miller, Rohan Taori, Aditi Raghunathan, Shiori Sagawa, Pang~Wei Koh,
  Vaishaal Shankar, Percy Liang, Yair Carmon, and Ludwig Schmidt.
\newblock Accuracy on the line: {O}n the strong correlation between
  out-of-distribution and in-distribution generalization.
\newblock In \emph{International Conference on Machine Learning}, 2021.

\bibitem[Mou et~al.(2022)Mou, Pananjady, and Wainwright]{mou2020optimal}
Wenlong Mou, Ashwin Pananjady, and Martin~J Wainwright.
\newblock Optimal oracle inequalities for solving projected fixed-point
  equations.
\newblock \emph{Mathematics of Operations Research}, 2022.

\bibitem[Munos(2003)]{munos2003error}
R{\'e}mi Munos.
\newblock Error bounds for approximate policy iteration.
\newblock In \emph{International Conference on Machine Learning}, 2003.

\bibitem[Munos(2007)]{munos2007performance}
R{\'e}mi Munos.
\newblock Performance bounds in ${L}_p$-norm for approximate value iteration.
\newblock \emph{SIAM Journal on Control and Optimization}, 2007.

\bibitem[Pathak et~al.(2022)Pathak, Ma, and Wainwright]{pathak2022new}
Reese Pathak, Cong Ma, and Martin Wainwright.
\newblock A new similarity measure for covariate shift with applications to
  nonparametric regression.
\newblock In \emph{International Conference on Machine Learning}, 2022.

\bibitem[Perdomo et~al.(2020)Perdomo, Zrnic, Mendler-D{\"u}nner, and
  Hardt]{perdomo2020performative}
Juan Perdomo, Tijana Zrnic, Celestine Mendler-D{\"u}nner, and Moritz Hardt.
\newblock Performative prediction.
\newblock In \emph{International Conference on Machine Learning}, 2020.

\bibitem[Quinonero-Candela et~al.(2008)Quinonero-Candela, Sugiyama,
  Schwaighofer, and Lawrence]{quinonero2008dataset}
Joaquin Quinonero-Candela, Masashi Sugiyama, Anton Schwaighofer, and Neil~D
  Lawrence.
\newblock \emph{Dataset {S}hift in {M}achine {L}earning}.
\newblock Mit Press, 2008.

\bibitem[Recht et~al.(2019)Recht, Roelofs, Schmidt, and
  Shankar]{recht2019imagenet}
Benjamin Recht, Rebecca Roelofs, Ludwig Schmidt, and Vaishaal Shankar.
\newblock Do imagenet classifiers generalize to imagenet?
\newblock In \emph{International Conference on Machine Learning}, 2019.

\bibitem[Ross et~al.(2011)Ross, Gordon, and Bagnell]{ross2011reduction}
St{\'e}phane Ross, Geoffrey Gordon, and Drew Bagnell.
\newblock A reduction of imitation learning and structured prediction to
  no-regret online learning.
\newblock In \emph{International Conference on Artificial Intelligence and
  Statistics}, 2011.

\bibitem[Sagawa et~al.(2020)Sagawa, Koh, Hashimoto, and
  Liang]{sagawa2019distributionally}
Shiori Sagawa, Pang~Wei Koh, Tatsunori~B Hashimoto, and Percy Liang.
\newblock Distributionally robust neural networks for group shifts: On the
  importance of regularization for worst-case generalization.
\newblock In \emph{International Conference on Learning Representations}, 2020.

\bibitem[Schmidt-Hieber and Zamolodtchikov(2022)]{schmidt2022local}
Johannes Schmidt-Hieber and Petr Zamolodtchikov.
\newblock Local convergence rates of the least squares estimator with
  applications to transfer learning.
\newblock \emph{arXiv:2204.05003}, 2022.

\bibitem[Shen et~al.(2021)Shen, Liu, He, Zhang, Xu, Yu, and
  Cui]{shen2021towards}
Zheyan Shen, Jiashuo Liu, Yue He, Xingxuan Zhang, Renzhe Xu, Han Yu, and Peng
  Cui.
\newblock Towards out-of-distribution generalization: A survey.
\newblock \emph{arXiv:2108.13624}, 2021.

\bibitem[Shimodaira(2000)]{shimodaira2000improving}
Hidetoshi Shimodaira.
\newblock Improving predictive inference under covariate shift by weighting the
  log-likelihood function.
\newblock \emph{Journal of Statistical Planning and Inference}, 2000.

\bibitem[Sugiyama and Kawanabe(2012)]{sugiyama2012machine}
Masashi Sugiyama and Motoaki Kawanabe.
\newblock \emph{Machine Learning in non-stationary environments: {I}ntroduction
  to covariate shift adaptation}.
\newblock MIT press, 2012.

\bibitem[Sugiyama et~al.(2007)Sugiyama, Nakajima, Kashima, Buenau, and
  Kawanabe]{sugiyama2007direct}
Masashi Sugiyama, Shinichi Nakajima, Hisashi Kashima, Paul Buenau, and Motoaki
  Kawanabe.
\newblock Direct importance estimation with model selection and its application
  to covariate shift adaptation.
\newblock \emph{Advances in Neural Information Processing Systems}, 2007.

\bibitem[Telgarsky(2021)]{mjt_dlt}
Matus Telgarsky.
\newblock Deep learning theory lecture notes.
\newblock \url{https://mjt.cs.illinois.edu/dlt/}, 2021.
\newblock Version: 2021-10-27 v0.0-e7150f2d (alpha).

\bibitem[Tsitsiklis and Van~Roy(1996)]{tsitsiklis1996analysis}
John Tsitsiklis and Benjamin Van~Roy.
\newblock Analysis of temporal-difference learning with function approximation.
\newblock \emph{Advances in Neural Information Processing Systems}, 1996.

\bibitem[Van~Roy and Dong(2019)]{van2019comments}
Benjamin Van~Roy and Shi Dong.
\newblock Comments on the {D}u-{K}akade-{W}ang-{Y}ang lower bounds.
\newblock \emph{arXiv:1911.07910}, 2019.

\bibitem[Xie and Jiang(2020)]{xie2020q}
Tengyang Xie and Nan Jiang.
\newblock ${Q}^\star$ approximation schemes for batch reinforcement learning:
  {A} theoretical comparison.
\newblock In \emph{Conference on Uncertainty in Artificial Intelligence}, 2020.

\bibitem[Xie and Jiang(2021)]{xie2021batch}
Tengyang Xie and Nan Jiang.
\newblock Batch value-function approximation with only realizability.
\newblock In \emph{International Conference on Machine Learning}, 2021.

\bibitem[Xie et~al.(2023)Xie, Foster, Bai, Jiang, and Kakade]{xie2022role}
Tengyang Xie, Dylan~J Foster, Yu~Bai, Nan Jiang, and Sham~M Kakade.
\newblock The role of coverage in online reinforcement learning.
\newblock In \emph{International Conference on Learning Representations}, 2023.

\bibitem[Yi and Neykov(2024)]{yi2024non}
Yufei Yi and Matey Neykov.
\newblock Non-asymptotic bounds for the ${L}_{\infty}$ estimator in linear
  regression with uniform noise.
\newblock \emph{Bernoulli}, 2024.

\bibitem[Yu and Bertsekas(2010)]{yu2010error}
Huizhen Yu and Dimitri~P Bertsekas.
\newblock Error bounds for approximations from projected linear equations.
\newblock \emph{Mathematics of Operations Research}, 2010.

\bibitem[Yu and Szepesv{\'a}ri(2012)]{yu2012analysis}
Yaoliang Yu and Csaba Szepesv{\'a}ri.
\newblock Analysis of kernel mean matching under covariate shift.
\newblock In \emph{International Conference on Machine Learning}, 2012.

\bibitem[Zhang et~al.(2022)Zhang, Backurs, Bubeck, Eldan, Gunasekar, and
  Wagner]{zhang2022unveiling}
Yi~Zhang, Arturs Backurs, S{\'e}bastien Bubeck, Ronen Eldan, Suriya Gunasekar,
  and Tal Wagner.
\newblock Unveiling transformers with lego: a synthetic reasoning task.
\newblock \emph{arXiv:2206.04301}, 2022.

\end{thebibliography}
